\documentclass[]{pkuai4m}

\usepackage[toc,page,header]{appendix}

\usepackage{minitoc}
\usepackage{multirow}     
\usepackage{multicol}     
\usepackage{booktabs}     
\usepackage{colortbl}     
\usepackage[table]{xcolor} 
\usepackage{makecell}     
\usepackage{graphicx}     
\usepackage{minted}
\usepackage{xcolor}
\usepackage{fontawesome5}
\usepackage{svg}
\usepackage{caption}      
\usepackage{fancyvrb}
\usepackage{fvextra}
\usepackage{enumitem}
\usepackage{wrapfig}
\usepackage{mathtools}
\usepackage{needspace}

\usepackage[most]{tcolorbox}
\usepackage{amsthm,amsmath,amssymb}

\theoremstyle{plain}
\newtheorem{theorem}{Theorem}
\newtheorem{corollary}[theorem]{Corollary}
\newtheorem{lemma}[theorem]{Lemma}
\newtheorem{proposition}[theorem]{Proposition}
\newtheorem{cited}[theorem]{Cited Result}
\newtheorem*{theorem*}{Theorem}
\newtheorem*{problem*}{Problem}
\newtheorem*{remark*}{Remark}
\newtheorem*{claim*}{Claim}
\newtheorem*{conjecture*}{Conjecture}

\theoremstyle{definition}
\newtheorem{definition}[theorem]{Definition}

\tcolorboxenvironment{theorem}{
    colback=blue!5,           
    colframe=blue!5,          
    arc=1mm,                  
    fonttitle=\bfseries,      
    boxrule=0pt,              
}
\tcolorboxenvironment{proposition}{
    colback=blue!5,           
    colframe=blue!5,          
    arc=1mm,                  
    fonttitle=\bfseries,      
    boxrule=0pt,              
}
\tcolorboxenvironment{lemma}{
    colback=blue!5,           
    colframe=blue!5,          
    arc=1mm,                  
    fonttitle=\bfseries,      
    boxrule=0pt,              
}
\tcolorboxenvironment{corollary}{
    colback=blue!5,           
    colframe=blue!5,          
    arc=1mm,                  
    fonttitle=\bfseries,      
    boxrule=0pt,              
}
\tcolorboxenvironment{definition}{
    colback=red!5,           
    colframe=red!5,          
    arc=1mm,                  
    fonttitle=\bfseries,      
    boxrule=0pt,              
}
\tcolorboxenvironment{problem*}{
    colback=orange!5,           
    colframe=orange!5,          
    arc=1mm,                  
    fonttitle=\bfseries,      
    boxrule=0pt,              
}
\tcolorboxenvironment{conjecture*}{
    colback=orange!5,           
    colframe=orange!5,          
    arc=1mm,                  
    fonttitle=\bfseries,      
    boxrule=0pt,              
}
\tcolorboxenvironment{cited}{
    colback=green!5,           
    colframe=green!5,          
    arc=1mm,                  
    fonttitle=\bfseries,      
    boxrule=0pt,              
}

\title{Automated Conjecture Resolution with Formal Verification}

\author[1,6,\dagger,*]{Haocheng Ju}
\author[1,6,\ddagger,*]{Guoxiong Gao}
\author[2,*]{Jiedong Jiang}
\author[3,11,*]{Bin Wu}
\author[1,4,*]{Zeming Sun}
\author[5,*]{Shurui Liu}
\author[1,6,11]{Leheng Chen}
\author[1,6]{Yutong Wang}
\author[1]{Yuefeng Wang}
\author[1,6]{Zichen Wang}
\author[1,6]{Wanyi He}
\author[6]{Peihao Wu}
\author[7]{Liang Xiao}
\author[7]{Ruochuan Liu}
\author[6,\P]{Bryan Dai}
\author[8,9,10,11\P]{Bin Dong}

\affiliation[]{
$^{1}$School of Mathematical Sciences, Peking University \\
$^{2}$Westlake Institute for Advanced Study, Westlake University \\
$^{3}$School of Mathematics, Tianjin University \\
$^{4}$Research Institute for Mathematical Sciences, Kyoto University \\
$^{5}$Department of Mathematics, Stanford University \\
$^{6}$IQuest Research \\
$^{7}$New Cornerstone Science Laboratory, School of Mathematical Sciences, Peking University \\
$^{8}$Beijing International Center for Mathematical Research and the New Cornerstone Science Laboratory, Peking University \\
$^{9}$Center for Machine Learning Research, Peking University\\
$^{10}$Center for Intelligent Computing, Great Bay Institute for Advanced Study, Great Bay University \\
$^{11}$Zhongguancun Academy \\
}

\contribution[*]{Equal Contribution}
\contribution[\dagger]{Informal Agent Lead}
\contribution[\ddagger]{Formal Agent Lead}
\contribution[\P]{Corresponding author}

\abstract{
Recent advances in large language models (LLMs) have significantly improved their ability to perform mathematical reasoning, extending from elementary problem solving to increasingly capable performance on research-level problems. However, reliably solving and verifying such problems remains challenging due to the inherent ambiguity of natural language reasoning. In this paper, we propose an automated framework for tackling research-level mathematical problems that integrates natural language reasoning with formal verification, enabling end-to-end problem solving with minimal human intervention. Our framework consists of two components: an informal reasoning agent, Rethlas, and a formal verification agent, Archon. Rethlas mimics the workflow of human mathematicians by combining reasoning primitives with our mathematical theorem search engine, Matlas, to explore solution strategies and construct candidate proofs. Archon, equipped with our formal theorem search engine LeanSearch, translates informal arguments into fully formalized Lean 4 projects through structured task decomposition, iterative refinement, and automated proof synthesis, ensuring machine-checkable correctness. Using this framework, we automatically resolve an open problem in commutative algebra proposed by D. D. Anderson (2014) and formally verify the resulting proof in Lean 4 with essentially no human involvement. Additional research-level case studies further illustrate the capabilities of Rethlas in informal mathematical reasoning and discovery, as well as the ability of Archon to formalize research-level proofs in Lean 4. Our experiments demonstrate that strong theorem retrieval tools enable the discovery and application of deep, cross-domain mathematical techniques, while the formal agent is capable of autonomously filling nontrivial gaps in informal arguments. More broadly, our work illustrates a promising paradigm for mathematical research in which informal and formal reasoning systems, equipped with theorem retrieval tools, operate in tandem to produce verifiable results, substantially reduce human effort, and offer a concrete instantiation of human–AI collaborative mathematical research with minimal human involvement.
}

\date{\today}

\def\emailicon{\raisebox{-1.5pt}{\includegraphics[height=1.05em]{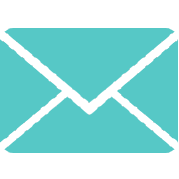}}}
\def\githubicon{\raisebox{-1.5pt}{\includegraphics[height=1.05em]{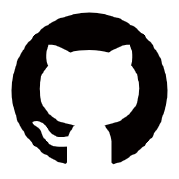}}}

\checkdata[ \emailicon \hspace{0.3em} Correspondence ]{\email{cbdai@iquestlab.com}, \email{dongbin@math.pku.edu.cn}}

\newcommand{\rethlaslink}{https://github.com/frenzymath/Rethlas}
\newcommand{\rethlasresultslink}
{https://github.com/frenzymath/Rethlas_results}
\newcommand{\archonlink}{https://github.com/frenzymath/Archon}
 \newcommand{\archonresultslink}{https://github.com/frenzymath/Anderson-Conjecture}
 
  \checkdata[ \githubicon \hspace{0.3em} Rethlas Source]{ \url{\rethlaslink} }
  \checkdata[ \githubicon \hspace{0.3em} Rethlas Results]{ \url{\rethlasresultslink} }
    \checkdata[ \githubicon \hspace{0.3em} Archon Source]{ \url{\archonlink} }
 \checkdata[ \githubicon \hspace{0.3em} Formalization Results ]{ \url{\archonresultslink} }

\begin{document}
\maketitle

\renewcommand{\thefootnote}{\fnsymbol{footnote}} 
\setcounter{footnote}{0}

\renewcommand{\thefootnote}{\arabic{footnote}}
\pagestyle{fancy}
\fancyhf{}
\fancyhead[R]{\thepage}


\section{Introduction}

In recent years, the mathematical reasoning abilities of large language models (LLMs) have advanced rapidly, progressing from handling elementary arithmetic and high-school–level problems to demonstrating competence across undergraduate curricula and beginning to engage with graduate- and research-level mathematics. Early models such as GPT-3 struggled with basic arithmetic, whereas GPT-4 \cite{achiam2023gpt} achieved a high score (92\%) on the grade-school benchmark GSM8K. With the emergence of \textit{test-time scaling}, where models allocate increased computational resources to reasoning during inference, systems such as OpenAI’s o1 have shown strong performance on high-school competition problems, including AIME. Subsequent reasoning models \cite{guo2025deepseek,qwq32b,yang2025qwen3,comanici2025gemini} have further validated the effectiveness of this paradigm. A notable milestone was achieved by Google’s Gemini Deep Think, which reached gold-medal–level performance at the International Mathematical Olympiad (IMO) using purely natural-language reasoning.

At a higher level, recent studies suggest that LLMs have also become increasingly capable in undergraduate and graduate mathematics, while continuing to make progress on research-level problems. For example, \cite{ju2026ai} showed that frontier models available at the time, such as Gemini 2.5 Pro, scored above 90 on a benchmark of undergraduate mathematics exams. The same study also found that o3-mini achieved an average score above 80 on PhD qualifying exams in Analysis, Probability, Algebra, and Geometry \& Topology. Open-source models have also demonstrated strong performance; for instance, \cite{jiang2025fate} reported that DeepSeek-R1 achieves 71.0\% proof accuracy on graduate-level algebra tasks. At the research level, progress is ongoing. On the \textit{FrontierMath} benchmark \cite{glazer2024frontiermath}, whose Tier 4 consists of unpublished research-level problems, o3-mini achieves only 4\% pass@1 accuracy, whereas GPT-5 improves this to 15\%. More recently, the leading system, GPT-5.4-Pro (web), reaches 38\% pass@1 accuracy. Notably, GPT-5.4-Pro has also solved an open problem in combinatorics from \textit{FrontierMath} that is categorized as “moderately interesting”. Building on these frontier models, LLM-based agents equipped with appropriate harness engineering can further enhance performance. For example, the \textit{Aletheia} agent \cite{feng2026towards}, which integrates a generator, a verifier, and a reviser, has autonomously tackled four Erdős problems \cite{feng2026semi} and addressed research questions such as computing eigenweights for the Arithmetic Hirzebruch Proportionality Principle \cite{feng2026eigenweights}, as well as problems concerning the simplicity of the Hodge bundle \cite{patel2026simplicity}.

However, despite these advances, evaluating the research-level capabilities of LLMs and LLM-based agents remains a significant bottleneck that requires substantial human effort. As mathematical complexity increases, evaluation increasingly relies on domain experts; however, due to the inherent imprecision of natural language, even experienced mathematicians can occasionally misjudge arguments. Indeed, errors can persist even within the rigorous peer-review processes of top journals. To enable fast and reliable verification of mathematical reasoning, it is therefore necessary to move toward more mechanized and unambiguous frameworks. Formal systems provide precisely such a foundation: a precise symbolic language paired with rigorously defined mechanisms for constructing and verifying proofs. Once a mathematical argument is translated into a formal language while preserving its semantic content, its correctness can be verified with complete rigor.

Motivated by this, we propose a framework for autonomously tackling and verifying research-level mathematics that integrates a natural language reasoning agent with a formalization agent, enabling end-to-end problem solving with minimal human supervision. The informal agent, Rethlas\footnote{Rethlas is open-sourced at \href{https://github.com/frenzymath/Rethlas}{https://github.com/frenzymath/Rethlas}.}, is designed to mimic the workflow of human mathematicians and is equipped with tools such as theorem retrieval to explore candidate solution paths. The formal agent, Archon \footnote{Archon is open-sourced at \href{https://github.com/frenzymath/Archon}{https://github.com/frenzymath/Archon}.}, builds on a dual-agent, tool-augmented architecture and is equipped with our formal theorem search engine LeanSearch \cite{gao2024semantic}. It transforms informal proofs into fully verified Lean 4 projects through structured planning, iterative formalization, and persistent memory management, ensuring both correctness and scalability.

Using this framework, we automatically resolve an open problem in commutative algebra proposed by D. D. Anderson in 2014 \cite{cahen2014open} and formally verify the resulting proof in Lean 4 with essentially no human intervention. During the natural-language solving phase, the semantic theorem search engine Matlas, employed by Rethlas, plays a crucial role in discovering a key technical result by Jensen \cite{jensen2006completions}. During formalization, Archon demonstrates its strong mathematical capability by filling non-trivial gaps left in the references and the natural language proof. The formalization has also passed the check of Comparator\footnote{https://github.com/leanprover/comparator}. 

We further demonstrate the capabilities of Rethlas and Archon separately on additional research-level problems. For informal problem solving, we test Rethlas on an expert-level problem in algebraic groups: a question that, to the best of our knowledge, had not been previously recorded in the published literature or online. Rethlas successfully proved the result, whereas GPT-5.5 Pro, accessed through its webpage interface, which also provides an agentic system, produced a completely incorrect proof. We also study a natural problem in $p$-adic Hodge theory to illustrate Rethlas’s ability to assist mathematical exploration and understanding. In this case, human mathematicians propose conjectural formulations, and Rethlas automatically proves the conjectures, shows that one assumption is unnecessary, sharpens a qualitative condition into a more explicit one, and constructs a counterexample when one condition is removed. This workflow helps mathematicians better understand the problem and goes beyond merely proving a well-formulated statement, suggesting its potential value for real mathematical research. Archon was tested on the formalization of two research-level problems and successfully formalized both: one fully autonomously, and the other with only a one-sentence natural-language hint.

The remainder of the paper is organized as follows. In \Cref{sec:related}, we review related work on agents for natural language and formal reasoning. \Cref{sec:overview} provides an overview of our framework. \Cref{sec:anderson_results} presents the automated resolution of Anderson's open problem, covering both Rethlas's natural-language proof and Archon's formalization in Lean~4. \Cref{sec:add_results} discusses the capabilities of the two agents separately on additional research-level problems. Finally, \Cref{sec:concl} concludes the paper.

\section{Related Work}\label{sec:related}
\subsection{Agents for Natural Language Reasoning}

LLM-based agents have recently demonstrated strong performance in both high-school olympiad mathematics and research-level mathematics. In the context of olympiad mathematics, a representative work is \cite{huang2025winning}, which proposes a verification-and-refinement pipeline that successfully solves 5 out of 6 problems from the IMO 2025 using models such as Gemini 2.5 Pro, Grok-4, and GPT-5. This significantly surpasses the baseline performance of these models without such a structured workflow.

At the research level, the \textit{Aletheia} agent \cite{feng2026towards}, built upon an advanced version of Gemini Deep Think, operates through an iterative process involving a generator, a verifier, and a reviser. It has tackled four Erdős problems in an autonomous manner \cite{feng2026semi}, as well as a generalized Erdős problem in a semi-autonomous setting with human collaboration \cite{barreto2026irrationality}. Beyond these, \textit{Aletheia} has also addressed genuine research problems, including computing eigenweights for the Arithmetic Hirzebruch Proportionality Principle \cite{feng2026eigenweights} and establishing bounds for independence sets \cite{lee2026lower}. Furthermore, \textit{Aletheia} solved 6 out of 10 problems in the FirstProof benchmark \cite{feng2026aletheia}, introduced in \cite{abouzaid2026first}, which consists of real research problems solved by human mathematicians but not publicly released at the time of evaluation. In parallel, the \textit{FullProof} system \cite{bryan2026motivic} demonstrates effective human–AI collaboration in algebraic geometry: the AI handles special cases, humans provide insights for the general case, and the AI subsequently completes the general solution based on these hints.

\subsection{Agents for Formal Reasoning}

Several agent systems targeting Lean~4 have recently achieved strong results on competition-level benchmarks. AlphaProof~\cite{hubert2025olympiad}, the pioneering RL-based agent from Google DeepMind, achieved silver-medal performance at IMO~2024. Seed-Prover~1.5~\cite{chen2025seed} trains tool-calling capabilities directly into the model and solves 11 of 12 Putnam~2025 problems. AxiomProver~\cite{breen2025ax} and the open-source Numina-Lean-Agent~\cite{liu2026numina} each solve all 12. Harmonic's Aristotle~\cite{achim2025aristotle}, which augments a fine-tuned proof-search model with MCTS and an informal reasoning pipeline, achieves IMO~2025 gold-medal performance and offers a public API to the community.

Beyond competition benchmarks, several of these systems have demonstrated research-level capabilities. Aristotle and AxiomProver have each formalized solutions to open Erd\H{o}s problems in Lean; AxiomProver has also solved and formalized open research conjectures such as Fel's conjecture on syzygies of numerical semigroups~\cite{chen2026fel}. With greater human involvement, Numina-Lean-Agent collaboratively formalized the Brascamp--Lieb theorem with two human experts~\cite{liu2026numina}, and the VML project~\cite{ilin2026semi} formalizes a result in mathematical physics with a single mathematician providing interactive guidance over ten days. At a larger scale, Math,~Inc.'s Gauss~\cite{mathinc_spherepacking} produces approximately 200{,}000 lines of Lean code to formally verify the Fields Medal---winning sphere packing proofs; for dimension~8, Gauss built upon a pre-existing human-prepared blueprint repository, while dimension~24 was autoformalized from the original paper. On the tooling side, Mistral's Leanstral~\cite{mistralai2026leanstral} offers an open-source agent designed for proof engineering in research-level formalization projects.

\section{Overview of the Framework}\label{sec:overview}

\begin{figure}[tb]
\centering
\includegraphics[width=0.99\linewidth]{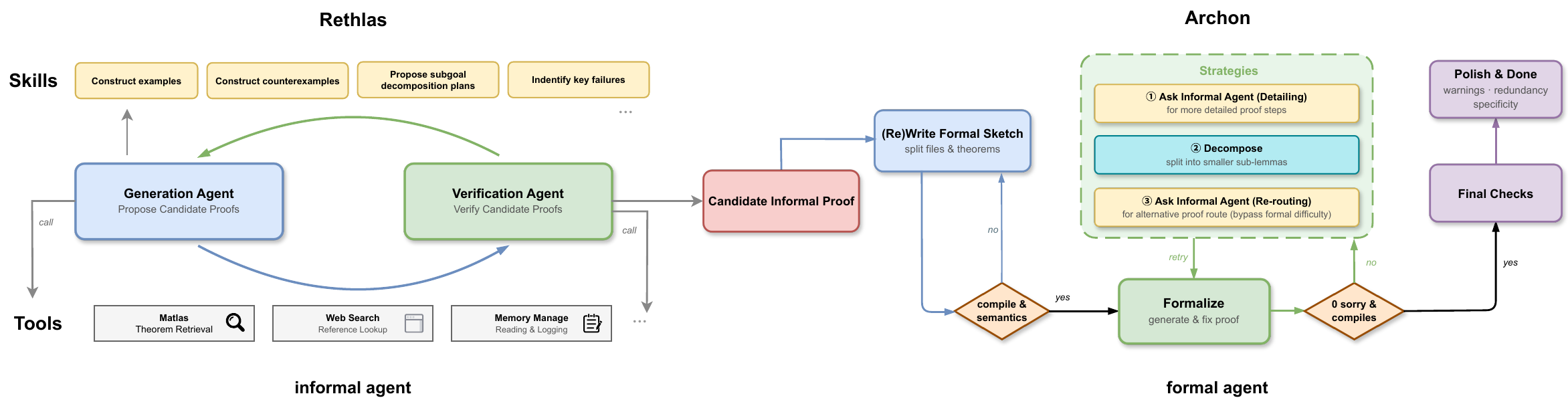}
\caption{Overview of the framework pipeline.}
\label{fig:complete}
\end{figure}

In this section, we describe our framework for autonomously tackling and verifying research-level mathematics. The framework consists of an informal agent (Rethlas), which first generates a candidate proof that is potentially correct, and a formal agent (Archon), which translates this informal proof into the formal language Lean 4 and fills in any gaps during the formalization process (see \Cref{fig:complete}). The design of the informal agent Rethlas is presented in \Cref{subsec:rethlas}, and the design of the formal agent Archon is presented in \Cref{subsec:archon}.

\subsection{Rethlas: The Informal Agent}\label{subsec:rethlas}

Rethlas consists of two subagents: a generation agent and a verification agent. As shown in \Cref{fig:rethlas}, it operates in an iterative process in which the generation agent produces an informal proof and passes it to the verification agent for correctness checking. If the proof passes verification, it is considered a candidate informal proof and the process terminates. Otherwise, the verification agent provides feedback, which is passed back to the generation agent to revise the proof.

\textbf{Generation agent.} The generation agent is designed to mimic the workflow of human mathematicians and is equipped with skills and tools tailored to mathematical research. These skills consist of reasoning primitives distilled from the processes mathematicians employ when tackling research problems, and they include:
\begin{itemize}
    \item \emph{Construct toy examples.} This skill is used when reasoning becomes stalled and simpler instances are needed to regain traction. The constructed examples should satisfy both the assumptions and the conclusion, allowing one to observe how the assumptions take effect and to develop intuition. When applying this skill, one should look for recurring patterns, invariants, or potential proof strategies suggested by the examples. Reasoning, decomposition, and retrieval may also be used to identify suitable examples or simplify the setting.
    \item \emph{Construct counterexamples.} This skill is used when a conjecture or intermediate claim appears fragile or insufficiently justified. The goal is to test whether the assumptions can hold while the conclusion fails. When applying this skill, one should leverage reasoning, decomposition, and retrieval to identify standard obstructions, pathological constructions, or known counterexamples.
    \item \emph{Search relevant results.} This skill is used to obtain theorems, constructions, examples, counterexamples, or background knowledge relevant to the current problem. It is also useful when constructing examples or counterexamples, or when proving subgoals that require supporting references. When a potentially useful theorem is found, one should expand the definitions and concepts involved using the surrounding context of the source, carefully verify its applicability to the current setting, and be explicit about any shifts in terminology across contexts. In addition, one should not stop at the statement alone, but also examine the proof to extract useful techniques, constructions, reductions, or proof patterns.
    \item \emph{Propose subgoal decomposition plans.} This skill is used after gathering sufficient insight from examples, counterexamples, search results, and prior attempts.
    \item \emph{Direct proving.} This skill is applied once one or more decomposition plans are available. Each plan is evaluated by attempting to prove its subgoals directly, while identifying key obstacles if the plan does not fully succeed.
    \item \emph{Recursive proving.} This skill is used when all current decomposition plans have been attempted via direct proving but none fully succeed. In this case, multiple subagents are deployed in parallel, each assigned a decomposition plan together with its associated key difficulties and those identified in other plans. Each subagent may refine, extend, or locally revise its assigned plan, while maintaining continuity rather than discarding it entirely.
    \item \emph{Identify key failures.} This skill is used when recursive attempts across all decomposition plans fail. The objective is to identify common patterns in these failures, such as recurring obstructions, ineffective decomposition strategies, or missing background knowledge. These insights are then summarized to guide the next iteration of plan generation.
\end{itemize}

\begin{figure}[tb]
\centering
\includegraphics[width=0.6\linewidth]{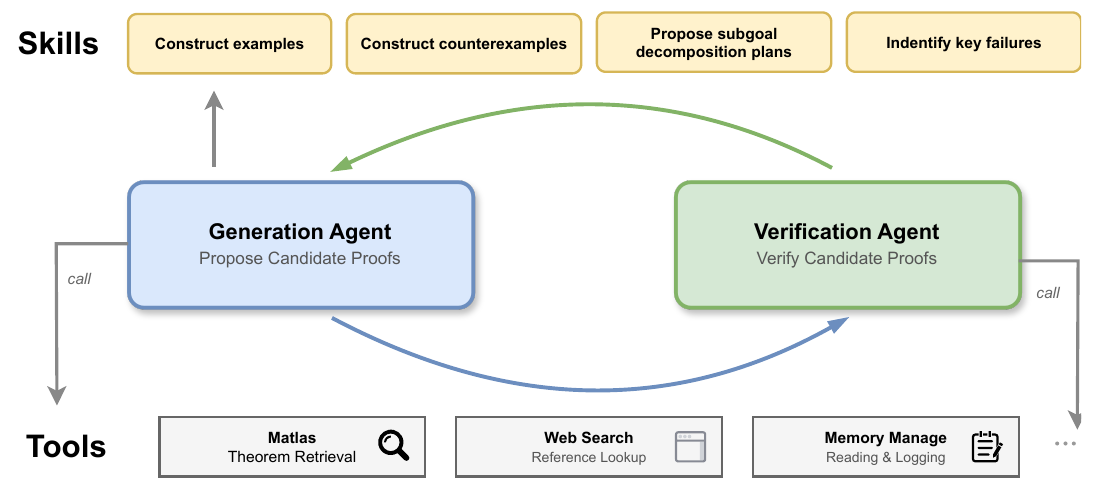}
\caption{Rethlas Agent.}
\label{fig:rethlas}
\end{figure}

As reflected in the above descriptions, these skills are not isolated but interconnected. Applying one skill often involves elements of others. For example, constructing examples and counterexamples does not occur in isolation; it typically requires reasoning, problem decomposition, and retrieval to identify appropriate constructions. When tackling a math problem, the agent is instructed to assess the current state and choose appropriate skills to apply, rather than deciding a fixed order of skills beforehand. 

Among the tools used in Rethlas, our theorem search engine Matlas \footnote{The version of Matlas used in the Rethlas experiments is a preliminary internal test system, implemented as an arXiv-based theorem search engine available at \href{https://leansearch.net/thm/}{https://leansearch.net/thm/}. Its corpus consists of approximately 13.6 million mathematical statements extracted directly from arXiv papers without further processing. Our latest Matlas system, available at \href{https://matlas.ai/}{https://matlas.ai/}, is a more advanced theorem search engine built from 8.51 million statements collected from 435K published papers across 180 journals and over 1,900 textbooks. As a result, it provides a more reliable and curated source of mathematical results. In addition, we extract statement dependencies and perform hierarchical unfolding to make statements more self-contained, which further improves their usability and quality compared to the earlier arXiv-based version. The previous arXiv-based theorem search engine will be deprecated in the near future.} plays a central role. In our experiments, Matlas serves as a semantic search engine over mathematical statements from arXiv, including definitions, propositions, theorems, corollaries, examples, and remarks. Its corpus contains approximately 13.6 million such statements, each embedded into a vector database. Given a query in the form of a mathematical statement, Matlas embeds the query and performs nearest-neighbor search using cosine similarity to retrieve relevant results. Compared to general web search, Matlas provides a more fine-grained and principled retrieval mechanism tailored to mathematical content. Although its corpus is smaller, it is more structured and enables more precise identification of relevant statements. When applying the \texttt{search-math-results} skill, Rethlas is instructed to first query Matlas and retrieve relevant theorems, examples, or counterexamples. It may then use web search either to gather additional background information and terminology, or to further explore references suggested by the results returned by Matlas.

In addition, Rethlas maintains a working memory of intermediate artifacts generated during the reasoning process, such as constructed examples, counterexamples, and subgoal decomposition plans. The agent is instructed to write these artifacts to memory and to query this memory when needed, enabling the reuse of previously generated insights and improving coherence across reasoning steps. In our experiments, unless otherwise stated, the generation agent is implemented using the OpenAI Codex agent with GPT-5.4 as the underlying model.

\textbf{Verification agent.} The verification agent is equipped with skills for checking the correctness of mathematical proofs. These include verifying statements sequentially, such as examining skipped or hand-wavy steps, identifying critical errors and gaps, and carefully assessing whether apparently unused assumptions are genuinely redundant or instead indicate an omitted argument. They also include confirming referenced statements, by checking, using our theorem search tool Matlas and web search, whether a cited statement exists and is applicable in the current context, expanding the definitions and terminology in the cited statement using the context of the source paper, and ensuring that terms used in the current proof carry the same meaning as in the cited source, since identical words may have different definitions in different contexts. In our experiments, unless otherwise stated, the verification agent is implemented using the OpenAI Codex agent with GPT-5.4 as the underlying model.

\subsection{Archon: The Formal Agent}\label{subsec:archon}

Mathematical proofs demand complete rigor, yet even expert-written proofs may contain subtle flaws, and proofs produced by LLMs, which are prone to hallucination, are far less reliable. To obtain trustworthy guarantees of correctness, we introduce formal verification into the pipeline. If a proof passes the formal language compiler, then verifying correctness reduces to checking that the top-level statement and its supporting definitions faithfully capture the intended mathematical content.

We adopt Lean 4 as our formal language. Its community-maintained library Mathlib comprises over 267,000 theorems and 127,000 definitions contributed by more than 770 members, providing a well-established foundation that makes formalization of frontier mathematics feasible, as every dependent definition and lemma must ultimately be grounded in existing libraries.

The formal agent, \textit{Archon}, extends our prior work on statement autoformalization~\cite{gao2024herald, wang2025aria} to full proof formalization: given an informal proof, it produces a Lean 4 project that states and proves the target theorem together with all supporting definitions and lemmas, grounded in Mathlib.
Archon first initializes the project by collecting dependent references, then formalizes key statements and definitions, organizing them into topic-specific files. It proceeds iteratively, filling gaps omitted in the literature, consulting the informal agent or external sources when stuck, and exploring alternative formalization strategies until the entire project compiles.

Prior to deployment on proofs generated by Rethlas, we validated Archon on FirstProof~\cite{abouzaid2026first}, a collection of 10 research-level problems without publicly released solutions at the time, thereby avoiding data contamination. Both OpenAI~\cite{openai_first_proof_2026} and Google~\cite{feng2026aletheia} evaluated their natural-language reasoning systems on these problems. We selected Problems~4 and~6 based on Mathlib infrastructure support and mathematical significance; notably, Google's Aletheia failed on both, while OpenAI succeeded but with human supervision. Archon produced complete formalizations of OpenAI's informal proofs for both problems: Problem~6 was formalized fully autonomously, and Problem~4 required only a single natural-language hint on the proof direction for one lemma. Details of these FirstProof experiments are presented in \Cref{subsubsec:archon-extra-results}. \Cref{subsubsec:workflow} introduces the workflow and architecture, while \Cref{subsubsec:enhancements} describes the enhancements made when moving from the FirstProof problems to the formalization of the conjecture.

\subsubsection{System Architecture and Workflow Design}\label{subsubsec:workflow}

\begin{figure}[!htbp]
    \centering
    \vspace{-0.6em}

    \begin{minipage}[t]{0.49\textwidth}
        \centering
        \vspace{0pt}
        \begin{minipage}[b][5.0cm][b]{\linewidth}
            \centering
            \includegraphics[
                width=\linewidth,
                height=5.0cm,
                keepaspectratio
            ]{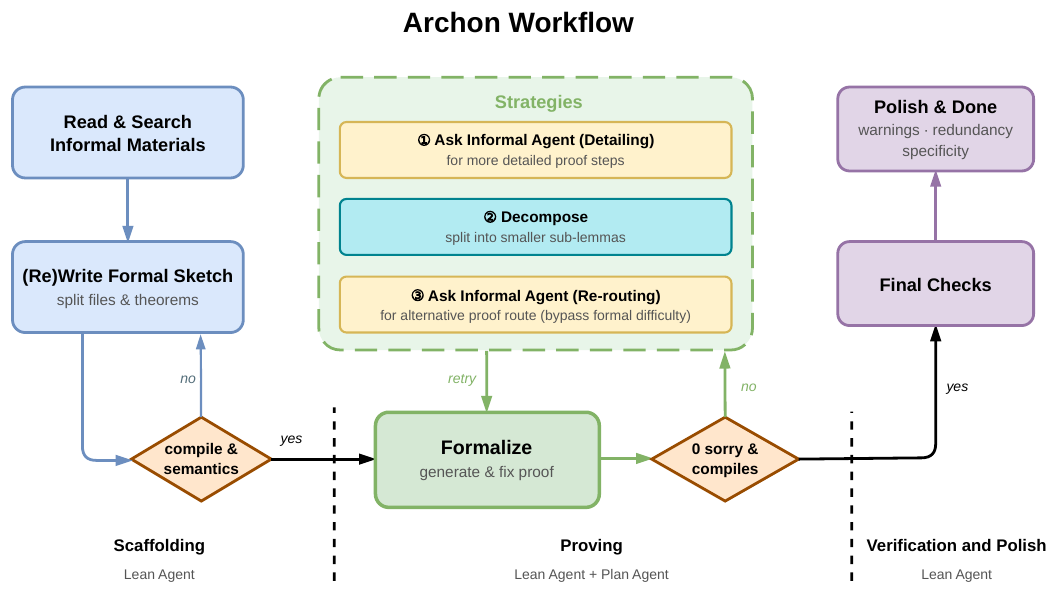}
        \end{minipage}
        \vspace{-1.8em}
        \captionof{figure}{Archon Workflow}
        \label{fig:archon-workflow}
    \end{minipage}\hspace{0.005\textwidth}%
    \begin{minipage}[t]{0.49\textwidth}
        \centering
        \vspace{0pt}
        \begin{minipage}[b][5.0cm][b]{\linewidth}
            \centering
            \includegraphics[
                width=\linewidth,
                height=5.0cm,
                keepaspectratio
            ]{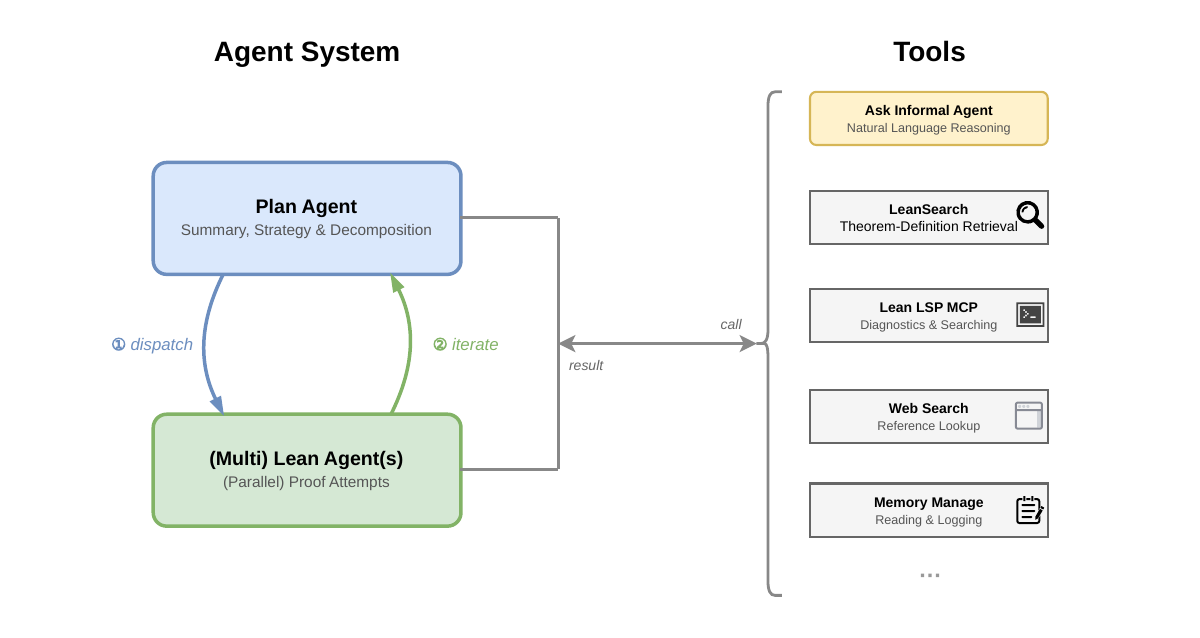}
        \end{minipage}
        \vspace{-1.8em}
        \captionof{figure}{Agent System and Tools}
        \label{fig:agent-system-tools}
    \end{minipage}

    \vspace{-0.6em}
\end{figure}

Our core design principle is to avoid hard-coded workflows and instead equip the agent with sufficient skills and tools, granting it maximal autonomy. However, we observed that in long-running sessions, particularly when the agent confronts difficult lemmas, accumulated context degrades performance: the agent's own failed attempts and pessimistic annotations erode its willingness to persevere. To address this, we introduced a dual-agent architecture consisting of a \textbf{Plan Agent}, which operates in a fresh context to decompose tasks and provide targeted guidance, and a \textbf{Lean Agent}, which executes formalization within a constrained scope. This separation substantially mitigates context pollution and task-aversion (e.g., context containing accumulated failures was observed to correlate with reduced agent persistence on difficult obligations).

The formalization process is divided into three phases (\Cref{fig:archon-workflow}):
\begin{enumerate}
    \item \textbf{Scaffolding.} The Lean Agent analyzes the informal proof and constructs the initial file structure: splitting the proof into modules, defining theorem signatures, and placing \texttt{sorry} placeholders at each proof obligation. This decomposition is not fixed and may be revised as the proof develops.
    \item \textbf{Proving.} The Lean Agent and Plan Agent enter an iterative cycle. The Lean Agent attempts to discharge \texttt{sorry} placeholders; when stuck, it returns diagnostics to the Plan Agent, which re-evaluates the proof strategy and dispatches a revised task. When the decomposition yields independent proof obligations, multiple Lean Agents can operate in parallel.
    \item \textbf{Verification and Polish.} The system confirms successful compilation and the absence of \texttt{sorry}, \texttt{axiom}, and other escape hatches, then performs a quality pass to extract reusable lemmas and reduce proof complexity.
\end{enumerate}

As shown in \Cref{fig:agent-system-tools}, the agents are equipped with a suite of tools implemented via terminal CLI and Model Context Protocols (MCPs). We highlight two that are particularly important to Archon's effectiveness:

\textbf{Ask Informal Agent.} When the Lean Agent requires mathematical guidance beyond what the provided informal proof covers, it delegates to an informal agent for natural-language sub-proofs or alternative strategies. In most cases, this is a lightweight Gemini API call for quick reasoning. When the agent encounters a more substantial obstacle, it may invoke Rethlas for deeper, long-horizon reasoning.

\textbf{LeanSearch.} Our previously published retrieval tool, widely adopted in the Lean community, provides fuzzy search over Mathlib's theorems and definitions. For Archon, we further improved its retrieval accuracy and updated it to the latest Lean and Mathlib versions. Crucially, LeanSearch enables the agent to reliably determine whether a needed result already exists in Mathlib, allowing it to make well-informed decisions about when to invoke library lemmas and when to formalize from scratch. This capability is essential for producing formalizations that fully leverage existing infrastructure.

Other tools include a Lean LSP MCP for compilation diagnostics, web search for retrieving published references, and a memory management system that persists the agent's architectural reasoning, failed approaches, and learned techniques across context compressions and session restarts.

Providing tools alone is not sufficient; the model also requires explicit guidance on when and how to deploy them. This behavioral guidance is encoded in \textbf{skills}, structured instruction files built upon the Lean 4 Skills framework that shape the agent's working patterns and corner-case handling to mirror authentic formalization practice.

\subsubsection{Enhancements for Conjecture Formalization\label{subsubsec:enhancements}}

The version of Archon used for FirstProof Problems~4 and~6 validated the dual-agent architecture: the Plan Agent largely eliminated cases where the Lean Agent stalled or refused to push forward. However, scaling to harder problems revealed two further challenges. First, as mathematical difficulty and project complexity grow, agents occasionally lose track of prior attempts and repeatedly explore the same dead ends, wasting computational resources. Second, as the project enlarges, agents must manage a growing body of informal references and can struggle to locate them when needed. We introduce two enhancements, both open-sourced, to address these issues.

\textbf{Memory system and Review Agent.} We shorten the duration of each Lean--Plan iteration and require each agent to log a summary of its session. Several global status documents are maintained across sessions, tracking the overall workflow stage, per-file proving status, and detailed records of local \texttt{sorry} elimination and refactoring work. Additionally, a \textbf{Review Agent} runs at each session boundary to synthesize trends across recent sessions. This multi-session perspective is essential for detecting stalls, which only manifest as patterns over consecutive sessions. While the Review Agent does not alter the formalization itself, it enables the Plan Agent to adjust strategy at the right time and gives human experts a clearer view of progress.

\textbf{Structured prompts and reference management.} We reorganize the agent's guidance and reference materials into a clear directory structure. During project initialization, human experts prepare paper references (with agent assistance), and the agent adds further references via search. When facing mathematical difficulties, the agent consults the informal agent, then records candidate proof routes in a dedicated location distinct from the original references. Skills and prompts, which govern the agent's working patterns and conventions, can be customized per project; in our experiments, explicitly instructing the Plan Agent to follow informal reference steps when they are deemed trustworthy significantly reduced time spent on unpromising alternatives. This design also keeps the context clean: agents load only the materials relevant to the current task, avoiding distraction from unrelated information.

\medskip

\noindent These design choices play a central role in the full-scale formalization described in \Cref{subsec:overview-result}.

\section{Automated Resolution of Anderson's Open Problem}\label{sec:anderson_results}

With the framework established, we now shift our focus to the target mathematical problem. \Cref{subsec:open_problem} introduces the background and Anderson's open problem. \Cref{subsec:nl_proof_sketch} presents an overview of the proof generated by Rethlas, and \Cref{subsec:discovery} traces the step-by-step discovery process that led to it. Finally, \Cref{subsec:overview-result} discusses Archon's formalization of the proof in Lean~4.

\subsection{Anderson's open problem}\label{subsec:open_problem}

Let \((R, \mathfrak{m})\) be a Noetherian local ring. Recall the following definitions:
\begin{enumerate}[leftmargin=*]
    \item $R$ is \emph{quasi-complete} if for every decreasing sequence $\{A_n\}_{n\ge 1}$ of ideals of $R$ and every integer \(k \ge 1\), there exists \(s_k\) such that  
    \[
    A_{s_k} \subseteq \Bigl(\bigcap_{n\ge 1} A_n\Bigr) + \mathfrak{m}^k.
    \]

    \item $R$ is \emph{weakly quasi-complete} if the above condition holds for all decreasing sequences with \(\bigcap_{n \ge 1} A_n = 0\); in this case the condition simplifies to  
  \[
  A_{s_k} \subseteq \mathfrak{m}^k.
  \]
\end{enumerate}

These properties arise naturally from the study of topologies on local rings. Quasi-completeness clearly implies weak quasi-completeness, and the implication that every complete local ring is quasi-complete was first proved by Chevalley \cite[Lemma 7]{chevalley1943theory}. D. D. Anderson posed the following question:

\begin{problem*}[{D. D. Anderson, \cite[Problem 8a]{cahen2014open}}]
Does weak quasi-completeness imply quasi-completeness for Noetherian local rings?
\end{problem*}

We tasked Rethlas with exploring both possibilities: whether weak quasi-completeness always implies quasi-completeness, or whether a counterexample exists. After about 45 minutes of autonomous reasoning, Rethlas established the following theorem, later formalized in Lean by Archon. This theorem gives a negative answer by constructing a counterexample.

\begin{theorem}\label{thm:main}
    There exists a weakly quasi-complete ring that is not quasi-complete.
\end{theorem}

\subsection{Overview of the mathematical proof}\label{subsec:nl_proof_sketch}

The counterexample discovered by Rethlas and formalized by Archon relies heavily on a technical result from a paper by Jensen \cite{jensen2006completions}, which was located by Matlas. This result is very challenging to uncover without domain expertise and is not obviously related at first glance. A body of research has been developed in this area, studying completions and generic formal fibers; see, for example, \cite{heitmann1993characterization, loepp1997constructing, jensen2006completions, fleming2019completely}. Nevertheless, Rethlas quickly identified the relevant result and successfully applied it.

For a detailed mathematical construction and proof, see Appendix~\ref{app:math_proof}; we present only a brief sketch here to clarify the logical structure and the role of Jensen's result. After initial reductions and the application of several lemmas from the related literature, it suffices to construct a ring \(A\) whose completion \(T = \widehat A\) satisfies the following conditions:
\begin{enumerate}[leftmargin=*]
\item[A.] The generic formal fiber of \(A\), i.e., \(\operatorname{Spec}\, (T \otimes_{A} \operatorname{Frac} A)\), is trivial.
\item[B.] There exists a prime element \(a \in A\) such that \(A/aA\) is a one-dimensional Noetherian local domain that is not analytically irreducible (equivalently, its completion with respect to the maximal ideal is not a domain).
\end{enumerate}

At this point, Jensen's result \cite[Corollary 2.4]{jensen2006completions} becomes relevant. Jensen completely characterizes, under mild assumptions, when a ring \(T\) can be realized as the completion of a local UFD \(A\) with trivial generic formal fiber. Setting
\[
T = \mathbb{C}[[x,y,z]]/(x^2 - yz)
\]
and verifying that \(T\) satisfies all conditions in Jensen's theorem yields the desired construction. Since \(T\) is not factorial and contains a nonprincipal height-one prime \(Q\), the contraction $q \coloneqq Q \cap A = (a)$ produces a quotient \(A/aA\) that is not weakly quasi-complete.

\subsection{Rethlas' discovery process}\label{subsec:discovery}

In this section, we present Rethlas' discovery process, illustrated in \Cref{fig:rethlas_exploration}, which traces the journey from the initial problem stated in \Cref{thm:main} to a complete proof.

\begin{figure}[hptb]
\centering
\includegraphics[width=\linewidth]{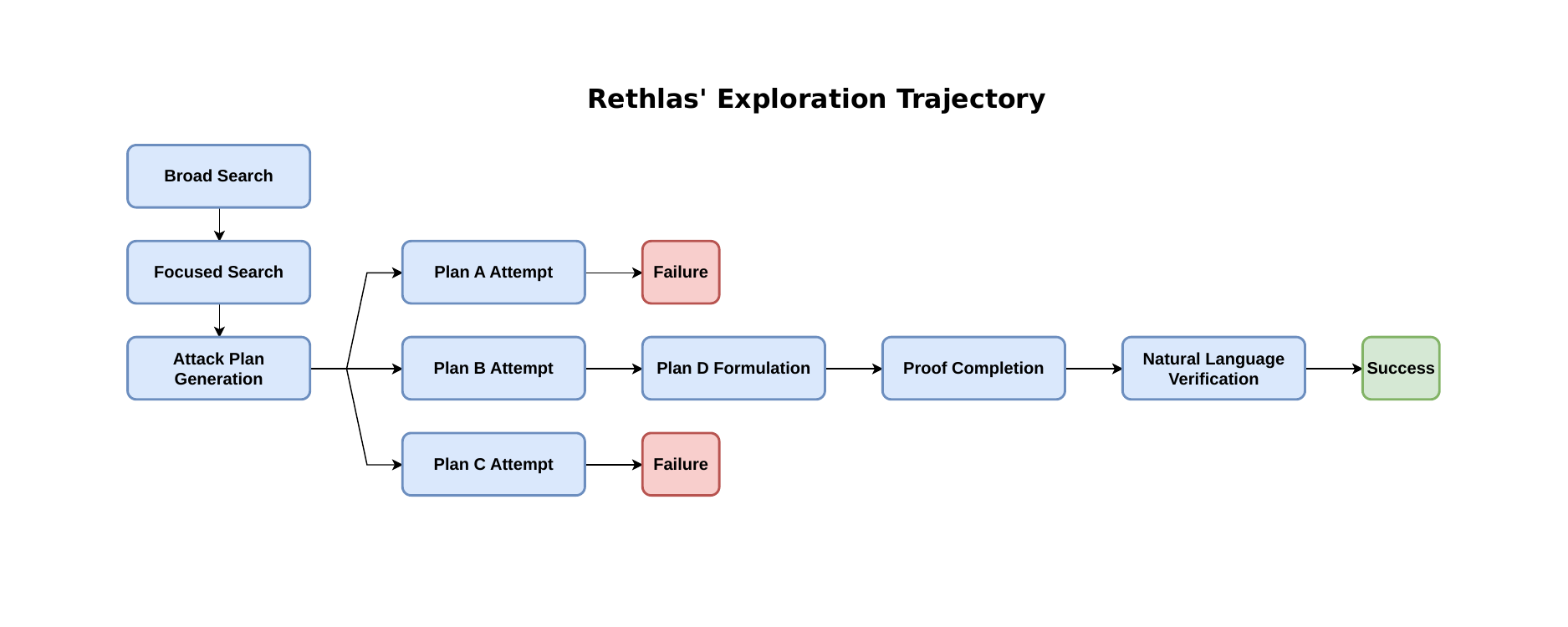}
\caption{Illustration of Rethlas' exploration trajectory on Anderson's open problem.}
\label{fig:rethlas_exploration}
\end{figure}

\begin{enumerate}[leftmargin=*]
    \item \textbf{Broad search.} By searching the original problem statement and its references, Rethlas located the literature that records this problem as an open question, as well as directly relevant technical sources including \cite{cahen2014open, farley2016quasi}. This led to a more tractable reformulation: it suffices to construct a weakly quasi-complete ring that admits a homomorphic image which is not weakly quasi-complete.
    \item \textbf{Focused search.} Searching the reformulated problem and related results from the previous step, Rethlas accumulated further technical understanding. During this phase, by searching the query \verb`There exists a` \verb`weakly quasi-complete local ring with a quotient isomorphic to k[X,Y]_(X,Y).` using Matlas, it discovered one of the works that systematically study the relationship between a Noetherian local ring and its completion, with particular attention to the generic formal fiber, namely \cite{fleming2019completely}. Rethlas noted that this line of research might be useful for constructing the desired counterexample.
    \item \textbf{Attack plan generation.} With the knowledge from the search results, Rethlas drafted three decomposition plans for attacking the problem.
    \begin{enumerate}[leftmargin=*]
        \item[A.] Construct a weakly quasi-complete ring \(R\) that surjects onto a known non-weakly-quasi-complete local ring \(S\) (ideally \(S = k[X,Y]_{(X,Y)}\)) via a square-zero extension or idealization.
        \item[B.] Construct a weakly quasi-complete local domain \(A\) using formal fiber control, then select a nonprime ideal \(I\) such that the completion of the quotient \(A/I\) exhibits an obvious obstruction to weak quasi-completeness.
        \item[C.] Begin with a non-weakly-quasi-complete local ring \(S\) and a weakly quasi-complete overring \(T\) with maximal ideal \(M\). Form a subring \(R = S + M\) of \(T\) or a fibered product, and test whether the added \(M\)-part forces every descending chain in \(R\) with zero intersection into powers of the maximal ideal.
    \end{enumerate}
    \item \textbf{Attempts of Plans A, B, C and formulation of Plan D.} Plans A and C failed after several attempts, yielding no clear progress. During these attempts, continued investigation of works relating a Noetherian local ring \(A\) to its completion \(T\) brought Jensen's result \cite{jensen2006completions} to light. Building on this, Rethlas identified a concrete candidate \(T = \mathbb{C}[[x,y,z]]/(x^2 - yz)\) and proposed Plan D, which is substantially clearer than Plan B.
    \begin{enumerate}[leftmargin=*]
        \item[D.] Use Jensen's corollary to construct a weakly quasi-complete local UFD \(A\) whose completion \(T\) is not factorial and contains a nonprincipal height-one prime \(Q\); then the contraction \(q = Q \cap A = (a)\) yields a quotient \(A/aA\) that is not weakly quasi-complete.
    \end{enumerate}
    \item \textbf{Proof completion and verification.} Rethlas worked out the remaining proof details and produced a Markdown file that clearly states all cited theorems and the logical steps of the argument. The natural-language verifier accepted the proof, yielding the final output. Some minor details that a human mathematician would consider negligible but that formal verification requires are not fully filled in; see Appendix~\ref{subsubsec:gap-filling}.
\end{enumerate}



\subsection{Archon's formalization process}\label{subsec:overview-result}

With the informal proof in hand, we now turn to its formalization and verification in Lean~4. The formalization\footnote{The complete formalization is open-sourced at \url{https://github.com/frenzymath/Anderson-Conjecture}.}, carried out by Archon, translates the complete proof of Anderson's conjecture into a Lean~4 project grounded in Mathlib, covering the main theorem, all supporting lemmas, and key results drawn from six external papers---Anderson~\cite{anderson2014quasi}, Farley~\cite{farley2016quasi}, Jensen~\cite{jensen2006completions}, Heitmann~\cite{heitmann1993characterization}, Loepp~\cite{loepp1997constructing}, and Chevalley~\cite{chevalley1943theory}. The task is non-trivial: beyond routine translation, it requires filling gaps left implicit in the informal argument, implementing underspecified constructions (such as transfinite recursion) in full detail, and adapting proof strategies to the available library infrastructure. The resulting codebase comprises approximately 19{,}000 lines of Lean~4 code across 42 files (see Figure~\ref{fig:code-structure} for a detailed breakdown) and was completed in approximately 80 hours of agent runtime.

\Needspace{0.45\textheight}
\begin{wrapfigure}{r}{0.38\textwidth}
    \centering
    \includegraphics[width=\linewidth]{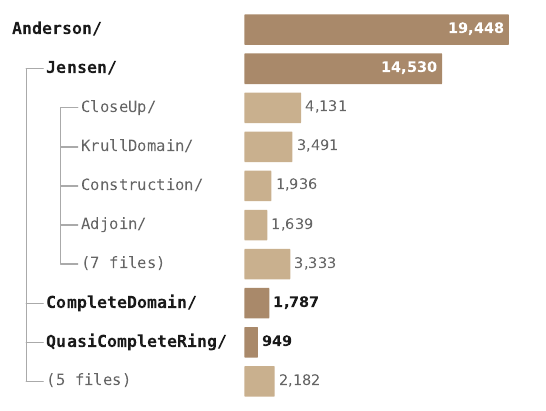}
    \vspace{-0.5\baselineskip}
    \caption{Structure of the formalization codebase (19{,}448 lines of Lean~4).}
    \label{fig:code-structure}
    \vspace{-0.5\baselineskip}
\end{wrapfigure}

The computational cost consisted of three Claude Code Max subscriptions (\$200/month each), with each account consuming roughly 70\% of its weekly quota during the one-week project. The entire process was fully autonomous: the only human intervention consisted of downloading paywalled PDF files that Archon could not retrieve on its own and placing them in the project's reference directory, after which Archon automatically performed OCR and organized the content into structured Markdown files for its own use; no mathematical judgment was required from the human operator. For context, based on interviews with Mathlib contributors, we estimate that an experienced formalization expert produces roughly 150--250 lines of Lean code per day; this codebase thus represents a workload comparable to several person-months of expert effort.

The correctness of the formalization is verified at multiple levels. The entire project passes \texttt{lake build} without errors. Additionally, we employ Comparator\footnote{\url{https://github.com/leanprover/comparator}}, a sandboxed verification tool for Lean proofs, to confirm that the formalized theorem statements match a human-reviewed simplified specification (Appendix~\ref{appendix:szm_comparator}) and that the proof relies only on the standard Lean axioms without introducing any hidden assumptions.

Beyond these quantitative metrics, the formalization process reveals several mathematical capabilities of Archon worth noting. First, Archon consistently demonstrated the ability to autonomously \textbf{fill non-trivial gaps} in the informal proof---arguments that the source text omits or leaves implicit---by supplying complete formal arguments without human input or additional informal context. Second, when an initial proof strategy proved untenable, Archon was able to diagnose its own mathematical errors and carry out a large-scale cross-session refactoring. Third, where the reference proof relied on definitions absent from Mathlib, Archon independently identified alternative proof strategies that bypass missing infrastructure. We give a fuller account of these capabilities, together with analogous behaviors observed on the FirstProof projects, in \Cref{subsec:archon-capabilities}.

These observations also suggest a practical mode of human--AI collaboration for formalization. In our experience, a mathematician can guide Archon in much the same way one explains a proof to a graduate student---by indicating key difficulties, pointing out errors, supplying references, and elaborating at specific bottlenecks---without needing to spell out every formal detail.

To quantify this collaborative mode, we conducted a controlled ablation\label{subsubsec:ablation-human}. At a checkpoint where three proof obligations remained open---all related to establishing that a certain intersection of subrings is a UFD (the same technical challenge discussed in Appendix~\ref{subsubsec:alternative-strategy})---we forked two branches from the same project state: in one, a human supervisor provided a Markdown proof blueprint (reproduced in Appendix~\ref{appendix:human-blueprint}) suggesting a Krull-domain route following Heitmann~\cite{heitmann1993characterization}; in the other, Archon continued without intervention. The human-guided branch resolved all remaining obligations in a single session (2 hours and 12 minutes), while the autonomous branch required two sessions and 3 hours and 43 minutes---roughly 70\% longer. Notably, Archon did not adopt the suggested Krull-domain route in the human-guided branch; it continued along the Kaplansky criterion approach it had already been developing (Appendix~\ref{subsubsec:alternative-strategy}), selectively incorporating only those intermediate observations from the blueprint that were applicable to both proof strategies. The primary effect of the human input was thus not to redirect the proof strategy but to supply useful intermediate results that accelerated the agent's progress.

\section{Capabilities of the Informal and Formal Agents}\label{sec:add_results}
In this section, we demonstrate the capabilities of our informal and formal agents on additional case studies. In \Cref{subsec:rethlas-capabilities}, we illustrate Rethlas’s ability in informal mathematical reasoning and discovery through examples in algebraic groups and $p$-adic Hodge theory. In \Cref{subsec:archon-capabilities}, we test Archon on the formalization of proofs for two research-level problems from FirstProof~\cite{abouzaid2026first}.

\subsection{Rethlas for Informal Mathematical Reasoning and Discovery}\label{subsec:rethlas-capabilities}
We demonstrate Rethlas’s ability to assist real mathematical research in two ways. In Section~\ref{subsubsec:rethlas-algebraic-groups}, we present an example in which Rethlas resolves a research problem in the theory of algebraic groups. We also note that Rethlas has resolved other documented open problems from the literature, including two problems from Ha{"i}m Brezis’s list of favorite open problems~\cite{Brezis2023Favorite}, as discussed in~\cite{dou2026degenerateconstantsdegreeinequalities}, and open problems in commutative algebra~\cite{jiang2026openproblemscommutativealgebra}; these examples are not discussed here, and we refer interested readers to the corresponding papers. In Section~\ref{subsubsec:rethlas-padic-hodge}, we present an example from $p$-adic Hodge theory in which Rethlas contributes to the exploratory process of identifying the correct mathematical picture: it proves a conjectural statement proposed by human mathematicians, clarifies which hypotheses are essential, sharpens a condition into an explicit bound, and constructs counterexamples when a key condition is removed. This example illustrates that Rethlas can support not only the proof of well-formulated mathematical statements, but also the broader process of mathematical exploration and discovery.

\subsubsection{Rethlas for Automated Problem Solving in Algebraic Groups}
\label{subsubsec:rethlas-algebraic-groups}

We illustrate Rethlas's capabilities in solving research-level problems by considering a natural and fundamentally important problem in the theory of algebraic groups which, to the best of our knowledge, was not recorded in the published literature or on the internet, and is comparable to problems in \cite{abouzaid2026first}. Rethlas, with GPT-5.5 xhigh used by the generation agent and GPT-5.4 xhigh used by the verification agent, produced a correct solution, except for a citation-substitution issue discussed in the following sections. We also tested GPT-5.5 Pro through its webpage interface, which provides an agentic system; it produced a completely wrong proof (see Appendix~\ref{app:gpt5.5pro-output}).


\textbf{Rational conjugacy classes of 1-parameter subgroups.}\label{sec: statement of brian's problem}

The theory of algebraic groups is of central importance in algebraic geometry, algebraic number theory, and representation theory. One-parameter subgroups play an important role in understanding the structure of algebraic groups. The following problem arises naturally in research.

Let $G$ be a connected smooth affine group scheme over a field $k$ and $\mathbb{X}_\ast(G)=\mathrm{Hom}_{k\textrm{-group}}(\mathrm{GL}_1, G)$ the set of 1-parameter subgroups of $G$ over $k$. There is a natural left action of $G(k)$ on $\mathbb{X}_\ast(G)$ by conjugation, i.e. 
$$(g\cdot \lambda)(t) = g\lambda(t)g^{-1},\quad\textrm{for $g\in G, \lambda\in\mathbb{X}_\ast(G), t\in \mathrm{GL}_1$}.$$
\begin{problem*}[Rational Conjugacy Classes of 1-parameter subgroups]
	Let $K$ be any field extension of $k$ and $G_K$ the base change of $G$ to $K$. Is $\mathbb{X}_\ast(G)/G(k)\rightarrow \mathbb{X}_\ast(G_K)/G(K)$ injective?
\end{problem*}
This is a natural and non-trivial problem in the theory of algebraic groups which, to the best of our knowledge, was not recorded in the literature or on the internet.

\begin{remark*} There are three easier cases with additional assumptions, which are known to experts but not recorded in the literature. 
\begin{enumerate}
	\item \textbf{Split case:} $G$ is assumed to be split reductive over $k$. This case is not hard and the proof strategy is obvious to experts.
	\item \textbf{Reductive case:} $G$ is assumed to be (not necessarily split) reductive (over a general field $k$).
	\item \textbf{Perfect field case:} $k$ is assumed to be a perfect field (and $G$ is a general smooth affine group scheme).
\end{enumerate}
All three easier cases above can be proved by both Rethlas (with GPT-5.4 xhigh) and GPT-5.5 Pro, accessed through the web interface (see \href{https://srliu3264.github.io/rethlas_results_math/}{Rethlas Results Github Homepage} for raw output).
\end{remark*}
We tasked Rethlas (with GPT-5.4 xhigh) to give proofs of the three easier cases above and then tasked Rethlas (with GPT-5.4 xhigh and GPT-5.5 xhigh) to give a proof or a counterexample for the general case, assuming the perfect field case is already known to be true. After running for 44 minutes, Rethlas proved the following theorem, giving an affirmative answer.
\begin{theorem}
	The natural map $\mathbb{X}_\ast(G)/G(k)\rightarrow \mathbb{X}_\ast(G_K)/G(K)$ is injective.	
\end{theorem}

For comparison, we tasked GPT-5.5 Pro, accessed through the web interface, the same way. It did not produce a meaningful solution; see Appendix~\ref{app:gpt5.5pro-output}. In contrast, Rethlas with GPT-5.4 xhigh produced an overall plausible solution, while Rethlas with GPT-5.5 xhigh produced a correct solution (except that, due to copyright constraints, it replaced a direct reference to \cite[Theorem C.2.15]{Conrad_Gabber_Prasad_2015} with a reference to the survey \cite[Theorem 5.3.2(iv)]{conrad_prasad2017pseudoreductive}, whose proof refers back to the same theorem but whose statement does not contain the required assertion.).

\textbf{Overview of the mathematical proof.}\label{sec: overview of brian's problem}

Rethlas factored the proof into five lemmas and a proof of the main theorem. The detailed proof discovered by Rethlas, using GPT-5.5 xhigh, is presented in Appendix~\ref{app:math_proof_brian}. We give only a brief sketch here and highlight the main points. \par

The logical structure is as follows:
\begin{enumerate}
\item \textbf{Reduce to the quasi-reductive group case}. Rethlas achieved this by proving the following lemmas. Let $U=R_{us,k}(G)$ be the maximal-split smooth connected unipotent normal subgroup of $G$. Then
\begin{enumerate}
\item The Galois cohomology set is trivial: $H^1(k,U)=1$.
\item The conjugacy of maximal split $k$-tori \cite[Theorem 4.2.9]{conrad_prasad2017pseudoreductive}, together with (a), shows that injectivity for $\overline{G}:=G/U$ implies injectivity for $G$. 
\item Then \cite[Corollary 3.12]{conrad2015solvable} implies that $\overline{G}$ is quasi-reductive, i.e. its $k$-unipotent radical is $k$-wound.
\end{enumerate}
\item \textbf{Prove the quasi-reductive group case}. This is the most technical part.
\begin{enumerate}
	\item Fix a maximal $k$-split torus $S$ of $G$. By \cite[Theorem 4.2.9, Proposition 5.3.1]{conrad_prasad2017pseudoreductive}, $\mathbb{X}_\ast(G)/G(k)$ is in bijection with $\mathbb{X}_\ast(S)/W(G,S)$, where $W(G,S) = N_G(S)/Z_G(S)$ is the relative Weyl group and $N_G(S)$ (resp., $Z_G(S)$) denotes the normalizer (resp., centralizer) of $S$.
	\item Similarly, after fixing a maximal $k$-torus $T$ of $G$ containing $S$, one obtains a bijection between $\mathbb{X}_\ast(G_{k_s})/G(k_s)$ and $\mathbb{X}_\ast(T_{k_s})/W(G_{k_s},T_{k_s})$.
	\item The relationship between the relative Weyl group $W(G,S)$ and the absolute Weyl group $W(G_{k_s},T_{k_s})$ is studied through the relationship between relative and absolute root systems. 
		\begin{itemize}
			\item The key ingredient is that the restriction map from absolute roots $\Phi_{\textrm{abs}}$ to relative roots ${}_k\Phi\cup\{0\}$ carries a root basis $\Delta_{\textrm{abs}}^+$ surjectively onto a root basis ${}_k\Delta_+$ or ${}_k\Delta_+\cup\{0\}$.
        \end{itemize}
\end{enumerate}
\end{enumerate}
It is worth mentioning that one of the key insights in this proof is to first reduce to the quasi-reductive case and then prove that case. The concept of quasi-reductive groups and their structure theory are developed in \cite{Conrad_Gabber_Prasad_2015}, but this reference is copyrighted and was not accessible to Rethlas with GPT-5.5 xhigh. 
However, Rethlas found the survey \cite{conrad_prasad2017pseudoreductive} of \cite{Conrad_Gabber_Prasad_2015}, and still identified the correct reduction to the quasi-reductive group case, which it called the ``wound unipotent case.'' The key ingredient in Step 2 is contained in \cite[Theorem C.2.15]{Conrad_Gabber_Prasad_2015}, to which Rethlas did not have access. Instead, it cited \cite[Theorem 5.3.2(iv)]{conrad_prasad2017pseudoreductive} (a survey article of \cite{Conrad_Gabber_Prasad_2015}); the proof of \cite[Theorem 5.3.2(iv)]{conrad_prasad2017pseudoreductive} in \textit{loc.cit.} refers back to \cite[Theorem C.2.15]{Conrad_Gabber_Prasad_2015} but the statement of \cite[Theorem 5.3.2(iv)]{conrad_prasad2017pseudoreductive} is partial and does not contain the required assertion.

\textbf{Rethlas's discovery process.}\label{sec: process of brian's problem}

In this section, we present Rethlas’s discovery process for the rational conjugacy problem for one-parameter subgroups, illustrated in \Cref{fig:rethlas_exploration_bc}.

\begin{figure}[hptb]
\centering
\includegraphics[width=0.9\linewidth]{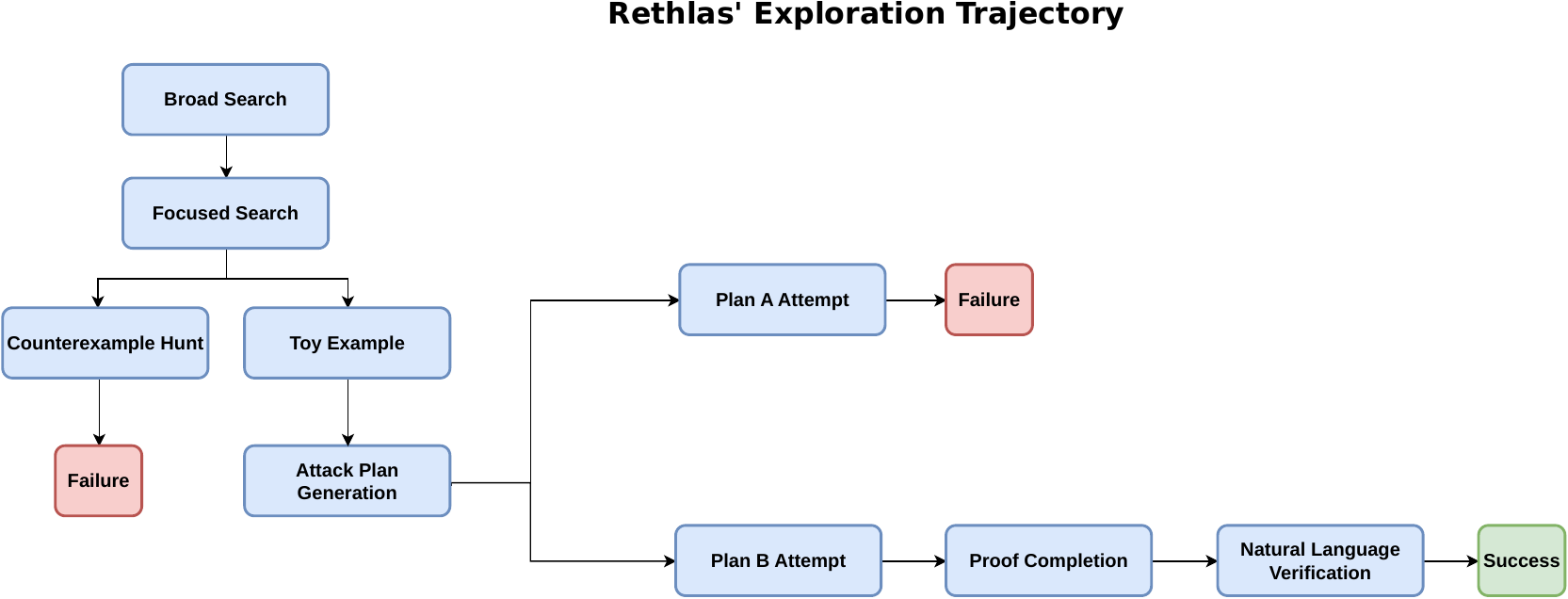}
\caption{Illustration of Rethlas' exploration trajectory on the algebraic group problem.}
\label{fig:rethlas_exploration_bc}
\end{figure}
\begin{enumerate}[leftmargin=*]
\item \textbf{Broad search.} 
    Rethlas conducted a broad search, filtered out weak references that assumed reductive groups or characteristic zero, and located relevant technical resources on algebraic groups, including \cite[Theorem 3.1.3]{lourencco2019grassmanniennes}, Brian Conrad’s answer to “Conjugate cocharacters in a maximal torus” on MathOverflow\footnote{\url{https://mathoverflow.net/questions/29073/conjugate-cocharacters-in-a-maximal-torus}}, \cite{conrad2015solvable}, and \cite{conrad_prasad2017pseudoreductive}. Rethlas also found \cite{hain2010remarks} and \cite{Rosengarten_2021}, aiming to use cohomological obstructions or pathological behavior of unipotent groups to find counterexamples.
\item \textbf{Focused search} 
	\begin{itemize}
    \item Rethlas searched for the structure of solvable groups and found \cite[Corollary 3.12, Corollary 4.4, Theorem 5.4, Corollary 5.12]{conrad2015solvable}.
    \item Rethlas searched for the Borel--Tits theorem and found \cite[Theorem 4.2.9]{conrad_prasad2017pseudoreductive} and \cite[Theorem 3.1.3]{lourencco2019grassmanniennes}. The two theorems are identical, and Rethlas decided to cite \cite{conrad_prasad2017pseudoreductive}.
    \item Rethlas searched for a key ingredient, namely the relation between absolute and relative roots, in particular \cite[Theorem 5.3.2]{conrad_prasad2017pseudoreductive}, and failed to download \cite{Conrad_Gabber_Prasad_2015} due to copyright restrictions.
\end{itemize}
	\item \textbf{Counterexample hunt and toy examples.} Rethlas failed to find a counterexample after a literature search and a study of cocharacters. In parallel, it analyzed the toy model
    \(G=\mathbb G_a\rtimes\mathbb G_m\), where \(\mathbb G_m\) acts on
    \(\mathbb G_a\) with positive weight $n$. Injectivity holds in this example. Rethlas therefore gave up the search for a counterexample and focused instead on proving injectivity.

    \item \textbf{Plan generation.} With these observations in memory, Rethlas
    proposed two main routes. 
  \begin{itemize}
	  \item Plan A (\verb|full_pseudo_reductive_quotient plan|, 7 lemmas and main theorem):
		 \begin{enumerate}
        \item Reduce to \(\overline{G} := G/R_{us,k}(G)\), by quotienting out the maximal \(k\)-split smooth connected unipotent normal subgroup of \(G\). This step uses two lemmas.
        \item Reduce to \(G_2=\overline{G}/R_{s,k}(\overline{G})\), where \(R_{s,k}(\overline{G})\) is the maximal central \(k\)-split torus, by \cite[Corollary 5.12]{conrad2015solvable}. This step uses two lemmas.
        \item Reduce to \(G_3=G_2/R_k(G_2)\), which is a pseudo-reductive group. This step is carried out in the proof of the main theorem.
        \item Prove the pseudo-reductive group case by studying relative roots and the relative Weyl group. This step uses three lemmas.
    \end{enumerate}

\item Plan B (\verb|split-then-wound plan|, 5 lemmas and main theorem): 
	\begin{enumerate}
        \item Reduce to the quasi-reductive case, which Rethlas called the wound unipotent case:
        \[
        \overline{G} := G/R_{us,k}(G),
        \]
        by quotienting out the maximal \(k\)-split smooth connected unipotent normal subgroup of \(G\). This step uses two lemmas.
        \item Prove the quasi-reductive case using relative roots and the relative Weyl group. This step uses three lemmas and the main theorem.
    \end{enumerate}
\end{itemize}
\item \textbf{Attempts of Plan A and B}
	\begin{itemize}
		\item Rethlas proved reduction to $\overline{G}$ by considering $H^1(k,U)$ and $\mathrm{GL}_1\times_{G/U}G$.
		\item Rethlas realized that the reduction to $G_2$ in Plan A might be tricky and recorded Plan A as a failure. Rethlas gave the following reason: "Unnecessary and riskier than the selected split-then-wound plan; central split-torus quotient step contains a potential torus-shear gap if written naively." As evidence, it noted: "Automorphisms of a split torus extension can act trivially on kernel and quotient while adding homomorphisms from quotient lattice to kernel lattice."
	\item For Plan B, Rethlas used information from the toy example above, realized that CGP relative root theory was the key ingredient, and finished all technical subgoals.
	\end{itemize}
   
    \item \textbf{Assembly and verification.} Rethlas worked out the remaining proof details and produced a Markdown file that clearly stated all cited theorems and the logical steps of the argument. The natural-language verifier accepted the proof, yielding the final output.
\end{enumerate}

\subsubsection{Rethlas for Mathematical Discovery in \texorpdfstring{$p$}{p}-adic Hodge Theory}
\label{subsubsec:rethlas-padic-hodge}

We present an example in which Rethlas assisted active research in $p$-adic Hodge theory; for a detailed mathematical exposition, we refer the reader to \cite{pan2026lift}. In mathematical practice, formulating the correct statement is often not much easier than proving it, and the two tasks are typically accomplished together.

\paragraph{Mathematical setup} Let $p$ be a prime and let $X$ be a proper smooth rigid curve over $C = \mathbb{C}_p$ with a smooth formal model $\mathfrak{X}/\mathcal{O}_C$. Fix a section $\exp : C \to 1 + \mathfrak{m}_C$ of the $p$-adic logarithm. The integral $p$-adic Simpson correspondence \cite{abbes2016p, min2024integral, anschutz2023small, sheng2024small} 
is an exact tensor equivalence
$$S_{\widetilde{\mathfrak{X}}} : \{\text{Hitchin-small } \widehat{\mathcal{O}}_X^+ \text{-bundles on } X_v\} \xrightarrow{\simeq} \{\text{Hitchin-small Higgs bundles on } \mathfrak{X}_{\textnormal{\'et}}\},$$
depending on a choice of a flat lifting $\widetilde{\mathfrak{X}}$ of $\mathfrak{X}$ over $A_2 = A_{\textnormal{inf}}^+/\xi^2$. A Higgs bundle $(H, \theta)$ is called \emph{lift-independent} if for any two such liftings $\widetilde{\mathfrak{X}}_1$, $\widetilde{\mathfrak{X}}_2$, there exists an isomorphism $S_{\widetilde{\mathfrak{X}}_1} (S_{\widetilde{\mathfrak{X}}_2}^{-1}(H, \theta)) \cong (H, \theta)$. A trivial observation is that Higgs bundles with zero Higgs field ($\theta = 0$) are lift-independent. A natural question is whether this is the only case: \emph{does every lift-independent Hitchin-small Higgs bundle have zero Higgs field?}

Unfortunately, this statement does not hold in general, though explicit calculations in small rank suggest it is true in most cases. The key is to find the correct formulation to prove. A key human insight replaces the lift-independence condition with a weaker but more explicit condition called \emph{cohomological lift-independence}, which simplifies the problem. Mathematicians first posed the following conjecture to Rethlas.
\begin{conjecture*}
\label{conj:padic-hodge-initial}
Suppose $\mathfrak{X}$ has genus $g \ge 2$. Let $(H, \theta)$ be a semistable, nilpotent Hitchin-small Higgs bundle of rank $r$ with $g \gg r$. If $(H, \theta)$ is cohomologically lift-independent, then $\theta = 0$.
\end{conjecture*}
All conditions added in the above conjecture are mathematically motivated. Semistability originates from the classical Simpson correspondence over the complex numbers. The assumptions $g \ge 2$ and $g \gg r$ are inspired by explicit calculations in small rank examples. Nilpotency is added to simplify the problem. This is an initial guess based on insights from mathematicians rather than a polished statement.

Rethlas produced a proof that does not use the nilpotency condition and replaces $g \gg r$ with the explicit bound $g \ge r^2 + 1$, thereby establishing the following strengthened theorem. The proof, together with its strengthened form, provides a partial explanation of why the condition $g \gg r$ is needed, which was previously unknown.
\begin{theorem}
\label{thm:padic-hodge-sharpened}
Suppose $\mathfrak{X}$ has genus $g \ge 2$. Let $(H, \theta)$ be a semistable Hitchin-small Higgs bundle of rank $r$ with $g \ge r^2 + 1$. If $(H, \theta)$ is cohomologically lift-independent, then $\theta = 0$.
\end{theorem}

Mathematicians also asked whether the condition $g \gg r$ can be removed. Rethlas found several counterexamples to the original conjecture when this condition is dropped. Below is one particularly interesting counterexample, derived from Landesman–Litt's solution \cite{landesman2024prill}
to Prill's problem in complex algebraic geometry. This counterexample might have taken mathematicians years to discover, as it originates from a distant topic.
\begin{proposition}
\label{prop:padic-hodge-counterexample}
For sufficiently large $p$, there exists a proper smooth rigid curve $X$ over $C = \mathbb{C}_p$ with a smooth formal model $\mathfrak{X}/\mathcal{O}_C$ of genus $g=2$, and a lift-independent Hitchin-small Higgs bundle $(H, \theta)$ of rank $r = 36$ whose Higgs field $\theta$ is nonzero (and not even nilpotent).
\end{proposition}

To conclude, in this example Rethlas assisted mathematicians with mathematical exploration and discovery in an active area of $p$-adic Hodge theory. The mathematicians proposed an insightful but immature conjecture; Rethlas not only provided a proof but also strengthened the statement by making conditions more explicit and exposing redundant hypotheses. Rethlas further drew a cross-domain counterexample from a distant topic to test the necessity of the assumptions. In this way, Rethlas collaborated with human mathematicians to simultaneously explore the correct formulation of the statement and find the proof, helping mathematicians gain a better understanding of the mathematics.

\subsection{Archon on Research-Level Formalization Tasks}\label{subsec:archon-capabilities}
We now turn to the capabilities of Archon, drawing on the Anderson formalization above together with our prior validation on the FirstProof benchmark~\cite{abouzaid2026first}. \Cref{subsubsec:archon-extra-results} presents additional formalization results that illustrate the scope of problems Archon can handle beyond Anderson's open problem. \Cref{subsubsec:archon-cap-list} synthesizes the notable mathematical capabilities of Archon observed across these projects. \Cref{subsubsec:archon-limitations} discusses recurring limitations and directions for further refinement.

\subsubsection{Additional formalization results}\label{subsubsec:archon-extra-results}

FirstProof~\cite{abouzaid2026first} is a benchmark of ten research-level mathematical problems whose reference solutions were not publicly available at the time of evaluation, mitigating concerns about data contamination. Two prior efforts published submissions on this benchmark: Google's Aletheia~\cite{feng2026aletheia} and OpenAI's first-proof technical report~\cite{openai_first_proof_2026}. From the ten problems, we selected Problems~4 and~6 for formalization based on two considerations: both sit within Mathlib's existing infrastructure to a degree that makes a focused formalization effort feasible, and both are proof-heavy problems whose core difficulty lies in symbolic reasoning rather than in algorithmic design or in discovering the right construction---a profile well suited to current formalization workflows. The remaining eight problems were excluded on one of these grounds: Problems~1, 2, 3, 5, and~7 require substantial domain-specific Mathlib infrastructure (stochastic PDEs, representation theory, equivariant homotopy theory, and symplectic geometry, among other topics) that is not yet in place, while Problems~8, 9, and~10 are discovery-heavy in the sense that a substantial part of the difficulty lies in algorithmic design or in identifying the right construction rather than in proof-heavy symbolic reasoning. Since OpenAI's results included these two problems, we ultimately chose them as our initial informal-language data.

The experimental setup for both problems was kept deliberately lightweight. A single human mathematician oversaw the runs, reviewing the agent's progress reports and intermediate Lean code approximately once per 24-hour cycle for roughly twenty minutes each time; no mathematical intervention was provided beyond the single one-sentence hint for Problem~4 described below. After Archon reported completion, we verified each formalization in two stages: an automated check that the compiled artifact is free of \texttt{sorry}, \texttt{axiom}, and other escape hatches, followed by a brief manual review (around five minutes per problem) of the top-level theorem statements and the key definitions on which they depend. Because Lean's kernel guarantees the remainder of the proof once these statements are accepted, this targeted review suffices to certify correctness; the same statement-verification methodology is discussed in more detail in \Cref{appendix:szm_comparator}.


Archon's formalization of Problem~6 was fully autonomous, while Problem~4 required a single one-sentence natural-language hint from a human supervisor at one obligation where Mathlib lacks the requisite analytic infrastructure (Appendix~\ref{subsubsec:firstproof4-vieta}); no other mathematical intervention occurred in either project. The full source code is available online.\footnote{The complete formalization is open-sourced at \url{https://github.com/frenzymath/Archon-FirstProof-Results}.} Problem~4 was completed in approximately 50~hours of agent runtime and Problem~6 in 30~hours; the two runs were carried out on per-call API usage, costing approximately \$1{,}200 and \$750 respectively, with all calls dispatched to Claude Opus 4.6.

\subsubsection{Notable mathematical capabilities}\label{subsubsec:archon-cap-list}

We highlight three categories of notable behavior observed during the Anderson formalization; analogous patterns also appeared in the FirstProof projects. Detailed mathematical examples are provided in Appendix~\ref{appendix:detailed-examples}.

\textbf{Autonomous gap-filling.} Archon reliably fills gaps that the informal proof leaves implicit, ranging from missing sub-arguments to entire constructions that the source literature describes only schematically. For instance, when the informal proof asserts an isomorphism without proving injectivity, Archon independently constructs the required argument at the coefficient level (Appendix~\ref{subsubsec:gap-filling}). At a larger scale, the transfinite recursive construction of the local UFD, described in the source papers only as ``recursively define $\{R_\alpha\}$ and take unions at limit ordinals'', required Archon to design a bundled record type, implement well-founded recursion on ordinals, and verify that algebraic invariants propagate through both successor and limit stages (Appendix~\ref{subsubsec:transfinite}).

\textbf{Self-diagnosis and cross-session refactoring.} When Archon's initial approach to the transfinite construction, based on Zorn's lemma and an unjustified countability assumption, failed, it diagnosed the root cause without human guidance: Zorn's lemma provides only existence without the fine-grained control the construction demands, and the countability assumption is not valid in ZFC without the Continuum Hypothesis. Archon then carried out a large-scale refactoring across multiple sessions, replacing the approach with explicit well-founded recursion and general cardinal arithmetic (Appendix~\ref{subsubsec:self-diagnosis} and~\ref{subsubsec:refactoring}).

\textbf{Discovery of alternative proof strategies.} When the reference proof relies on mathematical infrastructure absent from Mathlib, Archon can sometimes find alternative routes that bypass the gap entirely. In this project, a key lemma originally proved via Krull domain theory, which has no Mathlib formalization, was instead established using Kaplansky's criterion, a characterization that Archon identified independently and that does not explicitly appear in any of the reference papers used in this project (Appendix~\ref{subsubsec:alternative-strategy}). A similar episode arose in the formalization of FirstProof Problem~4: confronted with an application of the residue theorem to a rational function with polynomial numerator and denominator---whose standard formal proof requires toy contour theory, contour integration, and a formal notion of the interior of a closed curve, infrastructure not yet present in Mathlib---Archon recognized that the desired identity follows from the classical Euler--Jacobi formula and discharged the obligation through purely algebraic means (Appendix~\ref{subsubsec:firstproof4-residue}).

\subsubsection{Observed limitations}\label{subsubsec:archon-limitations}

We also observe several recurring limitations. These are tendencies rather than invariable behaviors, since Archon occasionally performs well in each area, but they arose frequently enough to merit discussion.

\textbf{Overthinking in the face of proof difficulties.} When confronted with proof obligations where the informal reference contains gaps or where Mathlib lacks the requisite infrastructure, Archon sometimes spends time deferring the obligation or searching for tangentially related material before committing to a direct attempt. While this behavior does not prevent eventual progress---Archon ultimately completed all obligations in our experiments---it reduces efficiency. We note a mild concern that for substantially more complex projects, such delays could accumulate and become a more serious bottleneck. In practice, providing targeted prompt-level guidance, such as instructing the agent to persist on a specific obligation, noticeably improves throughput. However, rigidly requiring Archon to follow the original proof is not always optimal: as illustrated by its independent discovery of the Kaplansky criterion route (Appendix~\ref{subsubsec:alternative-strategy}), the agent can sometimes identify simpler alternatives on its own. The most effective intervention is for a human formalization expert to provide case-by-case guidance, leveraging their judgment of which proof route is more tractable given the available Mathlib infrastructure.

\textbf{Code quality and Mathlib conventions.} Archon tends to produce proofs that are structurally verbose, relying heavily on \texttt{have} statements that mirror the step-by-step flow of natural-language reasoning. Individual proof blocks often span hundreds of lines without taking advantage of more idiomatic Lean tactics or term-mode constructions. When Archon extracts auxiliary lemmas to manage complexity, the results tend to be direct extractions of subgoals from the proof state---lightly simplified but not generalized into broadly applicable results. The naming conventions and statement styles also remain at a noticeable distance from Mathlib standards: theorem names do not consistently follow dot-notation patterns, and statement formulations sometimes differ from canonical forms. While none of these issues affect correctness, bringing the codebase into alignment with Mathlib conventions for potential upstream contribution would require non-trivial human effort.

\section{Conclusion}\label{sec:concl}
In this paper, we propose an automated framework that integrates natural language reasoning with formal verification to tackle research-level mathematical problems. The framework consists of two agents. The first is a natural language reasoning agent that mimics the workflow of human mathematicians and is equipped with our mathematical theorem search tool, Matlas. The second is a formal agent equipped with LeanSearch, consisting of two subagents: one responsible for task decomposition and targeted guidance, and the other for executing formalization. Using this framework, we successfully solved an open problem in commutative algebra \cite{cahen2014open} and automatically formalized the proof with essentially no human intervention. During the reasoning process, the success of the informal agent is largely attributed to Matlas, which enables the discovery of relevant theorems from adjacent fields and guides the search toward the correct direction. Meanwhile, the formal agent demonstrates the ability to autonomously fill nontrivial gaps that are omitted or left implicit in the informal proof, without requiring human assistance.

Beyond this end-to-end result, we further demonstrated the capabilities of Rethlas and Archon separately on additional research-level problems. On the informal-reasoning side, Rethlas handled an expert-level problem in algebraic groups which, to the best of our knowledge, had not previously appeared in the literature or on the internet. In contrast, GPT-5.5 Pro, accessed through the ChatGPT webpage rather than as a bare model call, produced a completely incorrect proof for the same problem. We also show that Rethlas can assist current mathematical research through a real research problem in $p$-adic Hodge theory. In this exploratory process, human mathematicians proposed insightful but immature conjectures, while Rethlas automatically proved them, strengthened the statements by making conditions more explicit and exposing redundant hypotheses, and constructed counterexamples when a key condition was removed. This case illustrates how Rethlas can help mathematicians better understand a problem and refine its correct formulation, going beyond merely proving an already well-formulated statement. On the formalization side, Archon successfully formalized natural-language proofs of two research-level problems, completing one entirely autonomously and the other with only a one-sentence hint. These case studies provide supporting evidence that the capabilities observed in the Anderson case are not confined to that example.

Our results highlight several key strengths of the framework. For the informal agent, it demonstrates strong capabilities in retrieval, understanding, and the application of deep domain-specific knowledge. In discovering the proof of Anderson's open problem, Rethlas explored a related branch connected to the open problem and successfully identified and applied highly technical methods (e.g., Jensen’s paper). Achieving such cross-domain integration would typically require collaboration between experts from different fields. This illustrates the effectiveness of retrieval tools in enabling a level of interdisciplinary reasoning comparable to expert-level discussions. In addition, the agent operates significantly faster than human mathematicians, completing tasks that would typically require substantial time for an individual mathematician unfamiliar with the relevant literature.

The framework also exhibits strong generality. The paradigm of discovering and leveraging knowledge from adjacent fields can be expected to extend to a wide range of mathematical domains and problems. For the formal agent, the framework greatly reduces human labor, requiring only one or two individuals with mathematical background and almost no expertise in formalization. Moreover, human mathematicians can interact with the system in a natural and intuitive way, similar to explaining concepts to students, by highlighting key difficulties, pointing out errors, suggesting references, and elaborating on challenging points, without the need for exhaustive formal detail.

Overall, our results demonstrate that, for a genuine open problem in mathematics, informal reasoning agents and formal agents can effectively cooperate: the informal agent generates plausible proofs, while the formal agent formalizes them and fills in missing details. Human involvement is limited to a minimal role, primarily verifying the semantic correctness of statements. This work provides a concrete example of how mathematical research can be substantially automated using AI.

\section*{Acknowledgements}
The authors would like to warmly thank Brian Conrad for bringing the algebraic group problem in Section~\ref{subsubsec:rethlas-algebraic-groups} to their attention, for valuable discussions on this problem, for carefully verifying Rethlas's output, and for his careful reading of and valuable feedback to Section~\ref{subsubsec:rethlas-algebraic-groups} and Appendix~\ref{app:math_proof_brian}. The authors also extend their warm thanks to Jiahong Yu for suggesting the 
$p$-adic Hodge theory problem in Section~\ref{subsubsec:rethlas-padic-hodge} and for carefully evaluating the output.

This work is supported in part by the National Key R\&D Program of China grant 2024YFA1014000, the Fundamental and Interdisciplinary Disciplines Breakthrough Plan of the Ministry of Education of China (JYB2025XDXM113), and the New Cornerstone Investigator Program.

\clearpage

\bibliographystyle{plainnat}
\bibliography{ref}

\clearpage

\beginappendix


\section{Mathematical Proof of Anderson's Open Problem}\label{app:math_proof}

In this section, we present the mathematical construction and proof of Anderson’s open problem as generated and verified by our complete automated pipeline. The original proof, produced by the AI system, has been refactored by human authors solely to improve exposition and clarity, without altering its logical content. This proof, together with several other commutative algebra problems solved by Rethlas, is also presented in the pure mathematical paper \cite{jiang2026openproblemscommutativealgebra}.

\subsection*{Definitions and statement of the main theorem}

Let $(R, M)$ be a Noetherian local ring. Recall the following definitions.

\begin{definition}
$R$ is \emph{quasi-complete} if for every decreasing sequence $\{A_n\}_{n=1}^{\infty}$ of ideals of $R$ and each integer $k \ge 1$, there exists $s_k$ such that 
\[
A_{s_k} \subseteq \Bigl(\bigcap_{n=1}^{\infty} A_n\Bigr) + M^k.
\]
\end{definition}

\begin{definition}
$R$ is \emph{weakly quasi-complete} if for every decreasing sequence $\{A_n\}_{n=1}^{\infty}$ of ideals of $R$ with $\bigcap_{n=1}^{\infty} A_n = 0$ and each integer $k \ge 1$, there exists $s_k$ such that 
\[
A_{s_k} \subseteq M^k.
\]
\end{definition}

\begin{definition}
$R$ is \emph{analytically irreducible} if its $M$-adic completion $\widehat{R}$ is an integral domain.
\end{definition}

Quasi-completeness can be viewed as a topological approximation property: every descending chain of ideals eventually becomes arbitrarily close to its intersection in the $M$-adic topology. The weak variant imposes this condition only for chains whose intersection is zero. The following basic relationships are known.

\begin{corollary}\label{cor:basic-relationships}
Let $(R, M)$ be a Noetherian local ring.
\begin{enumerate}[leftmargin=*]
\item If $R$ is complete with respect to the $M$-adic topology, then $R$ is quasi-complete \cite[Lemma 7]{chevalley1943theory}.
\item If $R$ is quasi-complete, then $R$ is weakly quasi-complete.
\item $R$ is quasi-complete if and only if every homomorphic image $R/I$ (for any proper ideal $I$) is weakly quasi-complete \cite[Theorem 1.3]{anderson2014quasi}.
\end{enumerate}
\end{corollary}

The main result of this section provides a negative answer to Anderson's question.

\begin{theorem}\label{thm:main-appendix}
There exists a weakly quasi-complete Noetherian local ring that is not quasi-complete.
\end{theorem}

\subsection*{Reduction of the problem}

By the third property in Corollary \ref{cor:basic-relationships}, a Noetherian local ring $R$ fails to be quasi-complete if and only if some homomorphic image $R/I$ is not weakly quasi-complete. Consequently, to prove \Cref{thm:main-appendix}, it suffices to construct a Noetherian local ring $A$ that is weakly quasi-complete yet admits a proper quotient failing weak quasi-completeness.

We now recall two external results that reduce this construction to concrete conditions.

\begin{cited}[{Farley, \cite[Proposition 1]{farley2016quasi}}]
A Noetherian local integral domain $A$ is weakly quasi-complete if and only if $P \cap A \neq \{0\}$ for every nonzero prime ideal $P$ of $\widehat A$.
\end{cited}

Equivalently, $A$ is weakly quasi-complete precisely when its generic formal fiber is trivial: the only prime of the completion $\widehat A$ lying over the zero ideal of $A$ is the zero ideal itself.

\begin{cited}[{Anderson, \cite[Corollary 2.2]{anderson2014quasi}}]

A one-dimensional Noetherian local domain is (weakly) quasi-complete if and only if it is analytically irreducible.
\end{cited}

Combining these criteria reduces our task to constructing a Noetherian local domain $A$ satisfying the following two conditions:

\begin{enumerate}[leftmargin=*]
\item[A.] The generic formal fiber of $A$ is trivial, which by Farley's criterion ensures that $A$ is weakly quasi-complete.
\item[B.] There exists a prime element $a \in A$ such that $A/aA$ is a one-dimensional Noetherian local domain that is not analytically irreducible, which by Anderson's criterion implies that $A/aA$ is not weakly quasi-complete.
\end{enumerate}

\subsection*{Construction of the complete local domain $T$}

\begin{lemma}\label{lem:complete_domain_choice-appendix}
Let
\[
T = \mathbb{C}[[x,y,z]]/(x^2 - yz),
\]
and let $M = (x,y,z)T$. Then:
\begin{enumerate}
\item $T$ is a complete $2$-dimensional Cohen–Macaulay local domain with $|T| = |T/M| = |\mathbb{C}|$.
\item The height-one prime ideal $Q = (x,y)T$ is not principal.
\end{enumerate}
\end{lemma}

\begin{proof}
The ring $\mathbb{C}[[x,y,z]]$ is a complete regular local domain of dimension $3$. The quotient
\[
T = \mathbb{C}[[x,y,z]]/(x^2 - yz)
\]
is isomorphic to the subring $\mathbb{C}[[u^2, uv, v^2]] \subseteq \mathbb{C}[[u,v]]$ via
\[
x \mapsto uv,\qquad y \mapsto u^2,\qquad z \mapsto v^2.
\]
Hence $T$ is a domain. Since $x^2 - yz$ is a nonzerodivisor in the regular local ring $\mathbb{C}[[x,y,z]]$, the quotient $T$ is a hypersurface and therefore Cohen–Macaulay of dimension $2$.

Moreover, $T/M \cong \mathbb{C}$, and a formal power series ring in finitely many variables over the uncountable field $\mathbb{C}$ has the same cardinality as $\mathbb{C}$; thus $|T| = |T/M| = |\mathbb{C}|$.

Finally,
\[
T/Q \cong \mathbb{C}[[z]],
\]
so $Q$ is prime and has height $1$. To see that $Q$ is not principal, note that $MQ \subseteq M^2$, while $x, y \notin M^2$ because the defining relation is quadratic. Hence the images of $x$ and $y$ in $Q/MQ$ are nonzero and linearly independent over $T/M \cong \mathbb{C}$. Therefore $\mu(Q) = \dim_{\mathbb{C}}(Q/MQ) \ge 2$, so $Q$ is not principal.
\end{proof}

The two properties of $T$ established above serve distinct roles in the proof. The cardinality condition $|T| = |T/M|$ satisfies a hypothesis of Jensen's construction, which will produce a ring $A$ with trivial generic formal fiber (condition A). The non-principality of $Q$ will be used to verify condition B.

\subsection*{Construction of the local UFD $A$}

\begin{lemma}\label{lem:jensen_special_case-appendix}
There exists a $2$-dimensional local UFD $A$ such that $\widehat{A}\cong T$ and the generic formal fiber of $A$ is local with maximal ideal $(0)$.
\end{lemma}

\begin{proof}
We apply the following external result.

\begin{cited}[{Jensen, \cite[Corollary 2.4]{jensen2006completions}}]
Let $(T,M)$ be a complete local ring with $|T/M| = |T|$, and let $P \in \operatorname{Spec} T$. Then there exists a local UFD $A$ such that $\widehat A = T$ and the generic formal fiber of $A$ is local with maximal ideal $P$ if and only if either
\begin{enumerate}
\item $T$ is a field or a DVR and $P = (0)$, or
\item $T$ has depth at least two and satisfies:
    \begin{enumerate}
        \item $P$ is nonmaximal and contains all associated primes of $T$;
        \item $P$ intersects the prime subring of $T$ trivially;
        \item if $J \in \operatorname{Spec} T$ satisfies $\operatorname{ht}(J) > \operatorname{depth}(T_J) = 1$, then $J \subseteq P$.
    \end{enumerate}
\end{enumerate}
\end{cited}

We apply this with the ring $T$ from Lemma \ref{lem:complete_domain_choice-appendix} and with $P=(0)$. The hypotheses hold:

\begin{enumerate}
\item $T$ is complete and $|T|=|T/M|$ by Lemma \ref{lem:complete_domain_choice-appendix}.
\item $T$ has depth~$2$ by Lemma \ref{lem:complete_domain_choice-appendix}.
\item Since $T$ is a domain, its only associated prime is $(0)$, so $P=(0)$ contains all associated primes.
\item $P=(0)$ is nonmaximal and meets the prime subring trivially.
\item Because $T$ is Cohen--Macaulay, there is no prime $J$ with $\operatorname{ht}(J)>\operatorname{depth}(T_J)=1$.
\end{enumerate}

Therefore Jensen’s corollary yields a local UFD $A$ with $\widehat A\cong T$ whose generic formal fiber is local with maximal ideal $(0)$.
\end{proof}

\subsection*{Verification that $A$ is the desired counterexample}

By Lemma \ref{lem:jensen_special_case-appendix}, the generic formal fiber of $A$ is trivial: the only prime ideal of $\widehat A\cong T$ contracting to $(0)$ in $A$ is $(0)$ itself. Hence every nonzero prime ideal of $\widehat A$ has nonzero contraction to $A$. By Farley’s criterion, $A$ is weakly quasi-complete. This establishes condition~A.

It remains to verify condition~B. Consider the height-one prime $Q=(x,y)T$ from Lemma \ref{lem:complete_domain_choice-appendix}. Since $Q\neq (0)$, the triviality of the generic formal fiber implies
\[
q:=Q\cap A \neq (0).
\]
Because $\widehat A$ is faithfully flat over $A$, we have
\[
1 \le \operatorname{ht}(q)\le \operatorname{ht}(Q)=1,
\]
so $\operatorname{ht}(q)=1$. As $A$ is a UFD, $q=aA$ for some prime element $a\in A$.

We claim that $aT$ is not prime. Indeed, $aT\subseteq Q$. If $aT$ were prime, then $\operatorname{ht}(aT)=1$, so the inclusion $aT\subseteq Q$ of height-one prime ideals would force $aT=Q$, contradicting the fact that $Q$ is not principal.

Therefore $T/aT$ is not a domain. Since completion commutes with quotient for Noetherian local rings,
\[
\widehat{A/aA}\cong \widehat A/a\widehat A \cong T/aT,
\]
so $\widehat{A/aA}$ is not a domain. Because $a$ is prime in the domain $A$, the quotient $A/aA$ is a one-dimensional Noetherian local domain. Hence $A/aA$ is not analytically irreducible. By Anderson’s criterion, $A/aA$ is not weakly quasi-complete. This establishes condition~B.

\subsection*{Conclusion}

The ring $A$ is weakly quasi-complete (condition~A), but its homomorphic image $A/aA$ is not weakly quasi-complete (condition~B). Since quasi-completeness is equivalent to every homomorphic image being weakly quasi-complete, $A$ is not quasi-complete. Thus $A$ is a weakly quasi-complete Noetherian local ring that is not quasi-complete, proving \Cref{thm:main-appendix}. \qed

\section{Mathematical Proof of the Algebraic Group Problem}\label{app:math_proof_brian}

In this section, we present the proof of the result on rational conjugacy classes of 1-parameter subgroups in \Cref{sec: statement of brian's problem} as generated and verified by our complete automated pipeline. The original proof, produced by the AI system, has been refactored by human authors solely to improve exposition and clarity, without altering its logical content.

\subsection*{Definitions and statement of the theorem}

Let $k$ be a field and let $G$ be a smooth connected affine $k$-group. We write
\[
\mathbb X_\ast(G):=\operatorname{Hom}_{k\text{-groups}}(\mathbf G_m,G)
\]
for the set of 1-parameter subgroups of $G$ over $k$. The group $G(k)$ acts on $\mathbb X_\ast(G)$ by conjugation:
\[
g\cdot \lambda=\operatorname{Int}(g)\circ\lambda.
\]
If $K/k$ is a field extension, then base change gives a natural map
\[
\mathbb X_\ast(G)/G(k)\longrightarrow \mathbb X_\ast(G_K)/G(K).
\]

Recall that $R_{u,k}(G)$ denotes the $k$-unipotent radical of $G$, and that
$
R_{us,k}(G)
$
denotes the maximal $k$-split smooth connected unipotent normal $k$-subgroup of $G$. A smooth connected unipotent $k$-group is called \emph{$k$-wound} if it contains no nontrivial $k$-split smooth connected unipotent subgroup.

\begin{theorem}\label{thm:cocharacter-conjugacy-descent}
Let $G$ be a smooth connected affine group over a field $k$, and let $K/k$ be an algebraic field extension. Then the natural map
\[
\mathbb X_\ast(G)/G(k)\longrightarrow \mathbb X_\ast(G_K)/G(K)
\]
is injective.
Equivalently, if two $k$-cocharacters of $G$ become conjugate over $K$, then they were already conjugate over $k$.
\end{theorem}
\begin{remark*}[From Human Experts]
The proof still works for general field extension $K/k$, i.e. one can drop the assumption that $K$ is algebraic over $k$ for this theorem and lemmas below (Lemma~\ref{lem:split-unipotent-reduction}, Lemma~\ref{lem:wound-unipotent-case} and Lemma~\ref{lem:wound-relative-dominant-representatives}) since the proofs did not use this assumption.
\end{remark*}
\subsection*{Split unipotent groups and rational points}

We first record a standard vanishing result for split unipotent groups.

\begin{lemma}\label{lem:split-unipotent-h1}
Let $U$ be a $k$-split smooth connected unipotent $k$-group. Then
\[
H^1(k,U)=1
\]
for nonabelian fppf cohomology. Consequently, if
\[
1\to U\to E\xrightarrow{\pi} Q\to 1
\]
is an exact sequence of smooth affine $k$-groups, then the map
\[
E(k)\to Q(k)
\]
is surjective.
\end{lemma}

\begin{proof}
Since $U$ is $k$-split, it admits a composition series by smooth connected normal $k$-subgroups whose successive quotients are $k$-isomorphic to $\mathbf G_a$. The result follows by induction on the length of such a series, using the exact sequence in nonabelian fppf cohomology and the standard vanishing
\[
H^1(k,\mathbf G_a)=0.
\]

For the asserted consequence, let $q\in Q(k)$. The fiber $\pi^{-1}(q)$ is a left $U$-torsor over $\operatorname{Spec} k$. Its class in $H^1(k,U)$ is trivial, so the fiber has a $k$-point. Thus $E(k)\to Q(k)$ is surjective.
\end{proof}

This lemma will allow us to quotient by the split unipotent radical without losing control of conjugacy over $k$.

\subsection*{Reduction modulo the split unipotent radical}

\begin{lemma}\label{lem:split-unipotent-reduction}
Let $G$ be a smooth connected affine $k$-group. Let $U=R_{us,k}(G)$
be its maximal $k$-split smooth connected unipotent normal $k$-subgroup, and put $\overline G=G/U.$
Assume that for every algebraic extension $K/k$, the natural map
\[
\mathbb X_\ast(\overline G)/\overline G(k)
\longrightarrow
\mathbb X_\ast(\overline G_K)/\overline G(K)
\]
is injective. Then for every algebraic extension $K/k$, the natural map
\[
\mathbb X_\ast(G)/G(k)
\longrightarrow
\mathbb X_\ast(G_K)/G(K)
\]
is injective.
\end{lemma}

\begin{proof}
Let $\pi:G\to \overline G$ be the quotient map. Suppose that
\[
\lambda,\mu\in \mathbb X_\ast(G)
\]
become $G(K)$-conjugate for some algebraic extension $K/k$.

After applying $\pi$, the cocharacters
\[
\overline\lambda:=\pi\circ \lambda,\qquad
\overline\mu:=\pi\circ \mu
\]
become $\overline G(K)$-conjugate. By the assumed injectivity for $\overline G$, there exists $\overline g\in \overline G(k)$ such that
\[
\overline\mu=\operatorname{Int}(\overline g)\circ \overline\lambda.
\]
By Lemma \ref{lem:split-unipotent-h1}, the map
\[
G(k)\to \overline G(k)
\]
is surjective. Choose $g\in G(k)$ lifting $\overline g$. Replacing $\lambda$ by $\operatorname{Int}(g)\circ\lambda$, we reduce to the case
\[
\pi\circ \lambda=\pi\circ \mu=:\nu:\mathbf G_m\to \overline G.
\]

Now form the pullback group
\[
E:=\mathbf G_m\times_{\nu,\overline G,\pi}G.
\]
Then $E$ is a smooth connected affine $k$-group fitting into an exact sequence
\[
1\to U\to E\xrightarrow{q}\mathbf G_m\to 1.
\]
The cocharacters $\lambda$ and $\mu$ are equivalently two $k$-homomorphic sections of $q$:
\[
s_\lambda(t)=(t,\lambda(t)),\qquad
s_\mu(t)=(t,\mu(t)).
\]

We use the following standard conjugacy theorem for split tori.

\begin{cited}[{\cite[Theorem 4.2.9]{conrad_prasad2017pseudoreductive}}]
Any two maximal split $k$-tori in a smooth connected affine $k$-group are conjugate under the group of $k$-rational points.
\end{cited}

The images
\[
s_\lambda(\mathbf G_m),\qquad s_\mu(\mathbf G_m)
\]
are split one-dimensional $k$-tori in $E$. They are maximal split $k$-tori. Indeed, any split torus in $E$ maps injectively to $\mathbf G_m$, because the kernel $U$ is unipotent and contains no nontrivial torus. Hence a split torus in $E$ has dimension at most $1$, and the two displayed tori are maximal.

By the cited theorem, there exists $e\in E(k)$ such that
\[
e\,s_\lambda(\mathbf G_m)\,e^{-1}=s_\mu(\mathbf G_m).
\]
Let
\[
a=q(e)\in k^\times=\mathbf G_m(k).
\]
Since $s_\mu(a)$ centralizes $s_\mu(\mathbf G_m)$, the element
\[
u:=s_\mu(a)^{-1}e
\]
still carries $s_\lambda(\mathbf G_m)$ onto $s_\mu(\mathbf G_m)$. Moreover,
\[
q(u)=1,
\]
so $u\in U(k)$.

Because $u\in \ker(q)$, we have
\[
q\circ \operatorname{Int}(u)\circ s_\lambda
=
q\circ s_\lambda.
\]
The restriction
\[
q|_{s_\mu(\mathbf G_m)}:s_\mu(\mathbf G_m)\to \mathbf G_m
\]
is an isomorphism. Therefore the only section of $q$ whose image is $s_\mu(\mathbf G_m)$ is $s_\mu$ itself. Hence
\[
s_\mu=\operatorname{Int}(u)\circ s_\lambda.
\]
Projecting to $G$ gives
\[
\mu=\operatorname{Int}(u)\circ \lambda.
\]
Thus $\lambda$ and $\mu$ are already $G(k)$-conjugate.
\end{proof}

Therefore, to prove the theorem for general $G$, it remains to prove it after quotienting by $R_{us,k}(G)$. In that quotient, the unipotent radical is $k$-wound.

\subsection*{A geometric normalizer lemma}

We next prove a purely geometric lemma over an algebraically closed field. It says that once two cocharacters inside a fixed maximal torus are conjugate by the whole group, they are already conjugate by the normalizer of that torus.

\begin{lemma}\label{lem:geometric-normalizer-conjugacy}
Let $F$ be an algebraically closed field, let $H$ be a smooth connected affine $F$-group, and let $T\subset H$ be a maximal torus. If
$
\lambda,\mu\in X_\ast(T)
$
are $H(F)$-conjugate, then they are conjugate by an element of $N_H(T)(F)$.
\end{lemma}

\begin{proof}
Choose $g\in H(F)$ such that
\[
\mu=\operatorname{Int}(g)\circ \lambda.
\]
Let
\[
C:=Z_H\bigl(\mu(\mathbf G_m)\bigr)^\circ.
\]
By the standard theorem on centralizers of tori in smooth affine groups, $C$ is a smooth connected affine $F$-group.

Both $T$ and $gTg^{-1}$ contain $\mu(\mathbf G_m)$, so both are contained in $C$. They are maximal tori of $C$: any larger torus in $C$ would be a larger torus in $H$.

Since $F$ is algebraically closed, maximal tori are maximal split tori. Applying the conjugacy theorem for maximal split tori to the smooth connected affine group $C$, we can choose $c\in C(F)$ such that
\[
c\,gTg^{-1}c^{-1}=T.
\]
Then
\[
n:=cg\in N_H(T)(F).
\]
Because $c$ centralizes $\mu(\mathbf G_m)$, we have
\[
\operatorname{Int}(n)\circ \lambda
=
\operatorname{Int}(c)\circ \mu
=
\mu.
\]
Thus $\lambda$ and $\mu$ are $N_H(T)(F)$-conjugate.
\end{proof}

\subsection*{The wound case and relative dominant representatives}

We now handle the key case: the unipotent radical is $k$-wound. The proof uses the relative root system attached to a maximal split torus.

\begin{lemma}\label{lem:wound-relative-dominant-representatives}
Let $H$ be a smooth connected affine $k$-group such that $R_{u,k}(H)$ is $k$-wound. Let $S\subset H$ be a maximal split $k$-torus, and let $P$ be a minimal pseudo-parabolic $k$-subgroup containing $S$.

Let
\[
{}_k\Phi=\Phi(H/R_{u,k}(H),S)
\]
be the relative root system, with positive system
\[
{}_k\Phi^+=\Phi(P/R_{u,k}(H),S).
\]
Define the closed dominant chamber
\[
C=
\left\{
\nu\in X_\ast(S)
\ \middle|\
\langle a,\nu\rangle\ge 0
\text{ for all }a\in {}_k\Phi^+
\right\}.
\]
Then:
\begin{enumerate}[leftmargin=*]
\item Every element of $X_\ast(S)$ is $N_H(S)(k)$-conjugate to a unique element of $C$.
\item If $\lambda,\mu\in X_\ast(S)$ become $H(K)$-conjugate over an algebraic extension $K/k$, then they have the same unique representative in $C$. In particular, $\lambda$ and $\mu$ are $N_H(S)(k)$-conjugate.
\end{enumerate}
\end{lemma}
\begin{remark*}[From Human Experts]
The closed chamber $C$ should be defined in the real vector space $X_*(S)_{\mathbb{R}}$ instead of $X_*(S)$, where $X_*(S)_{\mathbb{R}}$ is the direct sum of the $\mathbb{R}$-span of the roots $X_*(S')$ for $S' = (S \cap DH)_{\textrm{red}}^0)$ and $X_*(Z)$ for the maximal central $k$-split torus $Z$ in $H$.
\end{remark*}
\begin{proof}
We use the following structure results.

\begin{cited}[{\cite[Proposition 5.3.1]{conrad_prasad2017pseudoreductive}}]
If $S$ is a maximal split $k$-torus in a smooth connected affine $k$-group $G$, then
\[
W(G,S):=N_G(S)/Z_G(S)
\]
is a constant finite étale $k$-group, and the natural inclusion
\[
N_G(S)(k)/Z_G(S)(k)\hookrightarrow W(G,S)(k)
\]
is an equality.
\end{cited}

\begin{cited}[{\cite[Theorem 5.3.2(i)]{conrad_prasad2017pseudoreductive}}]
For a smooth connected affine $k$-group $G$, a maximal split $k$-torus $S$, and a minimal pseudo-parabolic $k$-subgroup $P$ containing $S$, the set
\[
{}_k\Phi=\Phi(G/R_{u,k}(G),S)
\]
is a root system in its $\mathbf Q$-span in $X(S)_\mathbf Q$; the subset
\[
\Phi(P/R_{u,k}(G),S)
\]
is a positive system of roots; and the natural map
\[
N_G(S)(k)/Z_G(S)(k)\to W({}_k\Phi)
\]
is an isomorphism.
\end{cited}

\begin{cited}[{\cite[Theorem 5.3.2(iv)]{conrad_prasad2017pseudoreductive}}]
If $R_{u,k}(G)$ is $k$-wound, then the root system $\Phi(G,S)$ consists precisely of the nontrivial $S$-weights on $\operatorname{Lie}(G)$.
\end{cited}

By the second cited result, the group
\[
N_H(S)(k)/Z_H(S)(k)
\]
identifies with the Weyl group of the relative root system ${}_k\Phi$. The usual root-system argument implies that every Weyl orbit on $X_\ast(S)$ contains a unique element of the closed dominant chamber $C$. By the first cited result, every Weyl element is represented by an element of $N_H(S)(k)$. This proves part (1).

For part (2), let $\lambda^+,\mu^+\in C$ be the unique dominant representatives of $\lambda$ and $\mu$ obtained in part (1). Since the conjugations from $\lambda$ to $\lambda^+$ and from $\mu$ to $\mu^+$ are already defined over $k$, the cocharacters $\lambda^+$ and $\mu^+$ are still $H(K)$-conjugate. Replacing $\lambda,\mu$ by $\lambda^+,\mu^+$, we may assume from now on that
$
\lambda,\mu\in C.
$

Let $k_s$ be a separable closure of $k$. Choose a maximal $k_s$-split torus
$
T\subset H_{k_s}
$
containing $S_{k_s}$. Choose a minimal pseudo-parabolic $k_s$-subgroup
$
B\subset P_{k_s}
$
containing $T$. Let
\[
\Phi_{\mathrm{abs}}
=
\Phi(H_{k_s}/R_{u,k_s}(H_{k_s}),T)
\]
be the corresponding absolute root system, with positive system
\[
\Phi_{\mathrm{abs}}^+
=
\Phi(B/R_{u,k_s}(H_{k_s}),T).
\]

For $\nu\in X_\ast(S)\subset X_\ast(T)$ and $\alpha\in \Phi_{\mathrm{abs}}$, we have
$
\langle \alpha,\nu\rangle
=
\langle \alpha|_S,\nu\rangle.
$
The $k$-wound property is preserved under separable extension. Hence, by the third cited result, the relative and absolute roots are precisely the nontrivial weights of $S$ and $T$ on the relevant Lie algebras.

Because $B\subset P_{k_s}$, every positive absolute root $\alpha\in \Phi_{\mathrm{abs}}^+$ restricts either to $0$ or to an element of ${}_k\Phi^+$. Conversely, every element of ${}_k\Phi^+$ occurs as the nonzero restriction of some element of $\Phi_{\mathrm{abs}}^+$. Therefore, for cocharacters belonging to $X_\ast(S)$, dominance with respect to ${}_k\Phi^+$ is equivalent to dominance with respect to $\Phi_{\mathrm{abs}}^+$.

\begin{remark*}[From Human Experts]
This is by applying \cite[Theorem C.2.15]{Conrad_Gabber_Prasad_2015} to $X_*(S')_{\mathbb{R}}$, rather than applying \cite[Theorem 5.3.2(iv)]{conrad_prasad2017pseudoreductive}.
\end{remark*}

Thus $\lambda$ and $\mu$ are also dominant for the absolute positive system $\Phi_{\mathrm{abs}}^+$.

Now let $\Omega$ be an algebraic closure containing both $K$ and $k_s$. Since $\lambda$ and $\mu$ are $H(K)$-conjugate, they are $H(\Omega)$-conjugate. The torus $T_\Omega$ is a maximal torus of $H_\Omega$. By Lemma \ref{lem:geometric-normalizer-conjugacy}, the cocharacters $\lambda$ and $\mu$ are conjugate by an element of
$
N_{H_\Omega}(T_\Omega)(\Omega).
$

By the first cited result applied over $k_s$ to the maximal split torus $T$, the group
$
N_{H_{k_s}}(T)/Z_{H_{k_s}}(T)
$
is a constant finite étale group. Therefore the purely inseparable extension from $k_s$ to $\Omega$ introduces no new Weyl elements. By the second cited result, this Weyl group is the Weyl group
$
W(\Phi_{\mathrm{abs}}).
$
Hence $\lambda$ and $\mu$ lie in the same $W(\Phi_{\mathrm{abs}})$-orbit.

But every Weyl-group orbit in a root system contains a unique element in the closed dominant chamber. Since both $\lambda$ and $\mu$ are $\Phi_{\mathrm{abs}}^+$-dominant, it follows that
$
\lambda=\mu.
$

Undoing the initial $N_H(S)(k)$-conjugations to dominant representatives, we conclude that the original $\lambda$ and $\mu$ have the same unique representative in $C$. Therefore they are $N_H(S)(k)$-conjugate.
\end{proof}

We can now prove injectivity in the wound case.

\begin{lemma}\label{lem:wound-unipotent-case}
Let $H$ be a smooth connected affine $k$-group such that $R_{u,k}(H)$ is $k$-wound. Then for every algebraic extension $K/k$, the natural map
\[
\mathbb X_\ast(H)/H(k)
\longrightarrow
\mathbb X_\ast(H_K)/H(K)
\]
is injective.
\end{lemma}

\begin{proof}
Let
$
\lambda,\mu\in \mathbb X_\ast(H)
$
be two $k$-cocharacters that become $H(K)$-conjugate.

Choose a maximal split $k$-torus $S\subset H$. Since the image of any $k$-cocharacter is a split $k$-torus, the conjugacy theorem for maximal split tori allows us to conjugate $\lambda$ and $\mu$, separately by elements of $H(k)$, so that both lie in
$
X_\ast(S).
$
These new cocharacters are still $H(K)$-conjugate.

Choose a minimal pseudo-parabolic $k$-subgroup $P$ containing $S$. Applying Lemma \ref{lem:wound-relative-dominant-representatives} to $H,S,P$, the two cocharacters in $X_\ast(S)$ are $N_H(S)(k)$-conjugate. Therefore the original cocharacters $\lambda$ and $\mu$ are $H(k)$-conjugate.
\end{proof}

\subsection*{Proof of the main theorem}

We now return to an arbitrary smooth connected affine $k$-group $G$.

Let
$
U=R_{us,k}(G)
$
be the maximal $k$-split smooth connected unipotent normal $k$-subgroup of $G$, and set
$
\overline G:=G/U.
$

We use one final structure theorem.

\begin{cited}[{\cite[Corollary 3.12]{conrad2015solvable}}]
For any smooth connected affine $k$-group $G$,
\[
R_{us,k}(G)=R_{u,k}(G)_{\mathrm{split}}.
\]
In particular,
\[
R_{u,k}(G)/R_{us,k}(G)
\]
is $k$-wound.
\end{cited}

Applying this to $G$, the unipotent radical of $\overline G$ is
\[
R_{u,k}(\overline G)
=
R_{u,k}(G)/R_{us,k}(G),
\]
which is $k$-wound. Therefore Lemma \ref{lem:wound-unipotent-case} applies to $\overline G$ and shows that, for every algebraic extension $K/k$, the map
\[
\mathbb X_\ast(\overline G)/\overline G(k)
\longrightarrow
\mathbb X_\ast(\overline G_K)/\overline G(K)
\]
is injective.

Finally, Lemma \ref{lem:split-unipotent-reduction} lifts this injectivity from $\overline G$ to $G$. Hence, if two $k$-homomorphisms
\[
\mathbf G_m\to G
\]
become conjugate by an element of $G(K)$, then they were already conjugate by an element of $G(k)$.

Thus the natural map
\[
\mathbb X_\ast(G)/G(k)
\longrightarrow
\mathbb X_\ast(G_K)/G(K)
\]
is injective. This proves \Cref{thm:cocharacter-conjugacy-descent}. \qed

\section{Raw Output of Anderson's Open Problem (Rethlas + GPT-5.4 xhigh)}

In this section, we present the raw output of Rethlas on Anderson's open problem. Only formatting adjustments were made to display the Markdown file in LaTeX.

\begin{Verbatim}[commandchars=!??]
# lemma lem:complete_domain_choice

## statement
Let
\[
T=\mathbb{C}[[x,y,z]]/(x^2-yz),
\]
and let \(M=(x,y,z)T\). Then:

1. \(T\) is a complete \(2\)-dimensional Cohen-Macaulay local domain.
2. \(|T|=|T/M|=|\mathbb{C}|\).
3. The height-one prime ideal
\[
Q=(x,y)T
\]
is not principal.

## proof
The ring \(\mathbb{C}[[x,y,z]]\) is a complete regular local domain of dimension \(3\). 
The quotient
\[
T=\mathbb{C}[[x,y,z]]/(x^2-yz)
\]
is isomorphic to the subring \(\mathbb{C}[[u^2,uv,v^2]]\subseteq \mathbb{C}[[u,v]]\) via
\[
x\mapsto uv,\qquad y\mapsto u^2,\qquad z\mapsto v^2.
\]
Hence \(T\) is a domain. Since \(x^2-yz\) is a nonzerodivisor in the regular local ring 
\(\mathbb{C}[[x,y,z]]\), the quotient \(T\) is a hypersurface and therefore Cohen-
Macaulay of dimension \(2\).

Also \(T/M\cong \mathbb{C}\), and a formal power series ring in finitely many variables 
over the uncountable field \(\mathbb{C}\) has the same cardinality as \(\mathbb{C}\), so 
\(|T|=|T/M|=|\mathbb{C}|\).

Finally,
\[
T/Q \cong \mathbb{C}[[z]],
\]
so \(Q\) is prime and has height \(1\). To see that \(Q\) is not principal, note that 
\(MQ\subseteq M^2\), while \(x,y\notin M^2\) because the defining relation is quadratic. 
Hence the images of \(x\) and \(y\) in \(Q/MQ\) are nonzero and linearly independent over 
\(T/M\cong\mathbb{C}\). Thus \(\mu(Q)=\dim_{\mathbb{C}}(Q/MQ)\ge 2\), so \(Q\) is not
principal.


# lemma lem:jensen_special_case

## statement
There exists a \(2\)-dimensional local UFD \(A\) such that \(\widehat{A}\cong T\) and 
the generic formal fiber of \(A\) is local with maximal ideal \((0)\).

## proof
We apply the following external result.

Complete cited statement:

Corollary 2.4 of Jensen, *Completions of UFDs with Semi-Local Formal Fibers* (paper_id: 
`10.1080/00927870500346321`, theorem_id: `Corollary 2.4`, arXiv id: none) states:

“Let \((T,M)\) be a complete local ring and \(|T/M|=|T|\). Let \(P\in \operatorname{Spec}
T\). Then there exists a local UFD \(A\) such that \(\widehat A=T\) and the generic 
formal fiber of \(A\) is local with maximal ideal \(P\) iff \(T\) is a field or DVR and 
\(P=(0)\), or \(T\) has depth at least two and:

1. \(P\) is nonmaximal and contains all associated prime ideals of \(T\);
2. \(P\cap\) the prime subring of \(T\) is \((0)\);
3. if \(J\in \operatorname{Spec}T\) with \(\operatorname{ht}(J)>\operatorname{depth}(T_J)
=1\), then \(J\subseteq P\).”

We apply this with the ring \(T\) from Lemma \ref{lem:complete_domain_choice} and with 
\(P=(0)\). The hypotheses hold:

1. \(T\) is complete and \(|T|=|T/M|\) by Lemma \ref{lem:complete_domain_choice}.
2. \(T\) has depth \(2\) by Lemma \ref{lem:complete_domain_choice}.
3. Since \(T\) is a domain, its only associated prime is \((0)\), so \(P=(0)\) contains 
all associated primes.
4. \(P=(0)\) is nonmaximal and meets the prime subring trivially.
5. Because \(T\) is Cohen-Macaulay, there is no prime \(J\) with \(\operatorname{ht}(J)>
\operatorname{depth}(T_J)=1\).

Therefore Jensen’s corollary yields a local UFD \(A\) with \(\widehat A\cong T\) whose 
generic formal fiber is local with maximal ideal \((0)\).


# lemma lem:a_is_weak_and_has_bad_quotient

## statement
For the ring \(A\) of Lemma \ref{lem:jensen_special_case}, the following hold:

1. \(A\) is weakly quasi-complete.
2. There exists a prime element \(a\in A\) such that \(A/aA\) is a one-dimensional 
Noetherian local domain that is not weakly quasi-complete.

## proof
Because the generic formal fiber of \(A\) is local with maximal ideal \((0)\), the only 
prime ideal of \(\widehat A\cong T\) contracting to \((0)\) is \((0)\) itself. Hence 
every nonzero prime ideal of \(\widehat A\) has nonzero contraction to \(A\).

Now invoke the following external result.

Complete cited statement:

Proposition 1 of Farley, *Quasi-completeness and localizations of polynomial domains: A
conjecture from “Open Problems in Commutative Ring Theory”* (paper_id: `BKMS.b140895`, 
theorem_id: `Proposition 1`, arXiv id: none) states:

“A Noetherian local integral domain \(R\) is weakly quasi-complete if and only if \(P\cap
R\neq\{0\}\) for each non-zero prime ideal \(P\) of \(\widehat R\), the completion of 
\(R\).”

Since \(A\) is a Noetherian local integral domain and every nonzero prime of \(\widehat 
A\) meets \(A\) nontrivially, Proposition 1 shows that \(A\) is weakly quasi-complete.

Next consider the height-one prime \(Q=(x,y)T\) from Lemma \ref{lem:complete_domain_
choice}. Since \(Q\neq (0)\), the triviality of the generic formal fiber implies
\[
q:=Q\cap A \neq (0).
\]
Because \(\widehat A\) is faithfully flat over \(A\), we have
\[
1 \le \operatorname{ht}(q)\le \operatorname{ht}(Q)=1,
\]
so \(\operatorname{ht}(q)=1\). As \(A\) is a UFD, \(q=aA\) for some prime element \(a\in 
A\).

We claim that \(aT\) is not prime. Indeed, \(aT\subseteq Q\). If \(aT\) were prime, then 
\(\operatorname{ht}(aT)=1\), so the inclusion \(aT\subseteq Q\) of height-one prime 
ideals would force \(aT=Q\), contradicting the fact that \(Q\) is not principal.

Therefore \(T/aT\) is not a domain. Since completion commutes with quotient for 
Noetherian local rings,
\[
\widehat{A/aA}\cong \widehat A/a\widehat A \cong T/aT,
\]
so \(\widehat{A/aA}\) is not a domain. Because \(a\) is prime in the domain \(A\), the 
quotient \(A/aA\) is a one-dimensional Noetherian local domain. Hence \(A/aA\) is not 
analytically irreducible.

Now use the criterion quoted in the problem statement itself:

Complete cited statement:

Anderson, *Quasi-complete semilocal rings and modules* (paper_id: `Anderson-2014-
Chapter`, theorem_id: `Corollary 2.2`, arXiv id: none): “A one-dimensional Noetherian 
local domain is (weakly) quasi-complete if and only if it is analytically irreducible.”

Since \(A/aA\) is a one-dimensional Noetherian local domain that is not analytically 
irreducible, it follows that \(A/aA\) is not weakly quasi-complete.


# theorem thm:main

## statement
Let (R,M) be a Noetherian local ring. R is said to be quasi-complete if for any decreasing
sequence {A_n}_{n=1}^!(!infty!) of ideals of R and each natural number k, there exists a 
natural number s_k with A_{s_k} !(!subseteq!) (!(!cap!)_{n=1}^!(!infty!) A_n) + M^k. If this condition holds 
for any decreasing sequence {A_n}_{n=1}^!(!infty!) of ideals of R with !(!cap!)_{n=1}^!(!infty!) A_n = 0, 
then R is called weakly quasi-complete (in which case we actually have A_{s_k} !(!subseteq!) M^k). 
If R is complete, then R is quasi-complete, which implies that R is weakly quasi-complete. 
Also, R is quasi-complete if and only if each homomorphic image of R is weakly quasi-
complete. The implication "R complete implies R is quasi-complete" was first proved by 
Chevalley [24, Lemma 7]; see also [4, Theorem 1.3]. Note that a DVR is quasi-complete 
but need not be complete. More generally, a one-dimensional Noetherian local domain is 
(weakly) quasi-complete if and only if it is analytically irreducible [4, Corollary 2.2]. 
References: [24] C. Chevalley, On the theory of local rings, Ann. Math. 44 (1943), 690–
708. [4] D. D. Anderson, Quasi-complete semilocal rings and modules, Commutative Algebra:
Recent Advances in Commutative Rings, Integer-Valued Polynomials, and Polynomial Functions, 
Springer Verlag, New York, 2014. Prove that there exists a weakly quasi-complete ring 
that is not quasi-complete.

## proof
Let \(A\) be the local UFD constructed in Lemma \ref{lem:jensen_special_case}. By Lemma 
\ref{lem:a_is_weak_and_has_bad_quotient}, \(A\) is weakly quasi-complete.

The same lemma also gives a prime element \(a\in A\) such that the homomorphic image 
\(A/aA\) is not weakly quasi-complete.

Finally, the problem statement itself records the equivalence:

Complete cited statement:

“A Noetherian local ring \(R\) is quasi-complete if and only if each homomorphic image 
of \(R\) is weakly quasi-complete.”

Applying this to \(A\), the quotient \(A/aA\) shows that \(A\) is not quasi-complete.

Thus \(A\) is a weakly quasi-complete Noetherian local ring that is not quasi-complete. 
This proves the required existence statement.
\end{Verbatim}

\section{Output for the Anderson's Open Problem (GPT-5.4 Pro)}

Here is a concrete example.

Take a prime \(p\), and let
\[
k=\mathbf F_p(u_1,u_2,u_3,\dots)
\]
be the rational function field in countably many indeterminates. Put
\[
A:=k^p[[X,Y]][k]\subset k[[X,Y]].
\]
Nagata's appendix \cite{NagataLocalRings} construction shows that for rings of the form
\(k^p[[X_1,\dots,X_n]][k]\), the ring is a regular local ring and is
properly contained in its completion \(k[[X_1,\dots,X_n]]\); we use the
case \(n=2\).

So \(A\) is a Noetherian local ring, and its \(M\)-adic completion is
\[
\widehat A = k[[X,Y]].
\]

I will prove:
\begin{enumerate}
    \item \(A\) is weakly quasi-complete.
    \item \(A\) is not quasi-complete.
\end{enumerate}

The second step will use the criterion you quoted: a Noetherian local ring
is quasi-complete if and only if every homomorphic image is weakly
quasi-complete \cite{cahen2014open}.

\subsection*{A useful criterion}

Let \((R,\mathfrak m)\) be a Noetherian local ring and \(\widehat R\) its
\(\mathfrak m\)-adic completion.

\begin{lemma}\label{lem: weak_qc_crt}
If every nonzero prime ideal of \(\widehat R\) contracts to a nonzero ideal
of \(R\), then \(R\) is weakly quasi-complete.
\end{lemma}

\begin{proof}
Assume \(R\) is not weakly quasi-complete. Then there is a decreasing sequence
\[
I_1\supseteq I_2\supseteq \cdots
\]
with \(\bigcap_n I_n=0\), and some \(k\ge 1\), such that
\[
I_n\nsubseteq \mathfrak m^k \qquad \text{for all } n.
\]

For each \(r\ge k\), consider
\[
J_{n,r}:=(I_n+\mathfrak m^r)/\mathfrak m^r \subseteq R/\mathfrak m^r.
\]
Since \(R/\mathfrak m^r\) is Artinian, the descending sequence
\(J_{1,r}\supseteq J_{2,r}\supseteq\cdots\) stabilizes; let its eventual
value be \(J_r\). Because \(I_n\nsubseteq \mathfrak m^k\) and \(r\ge k\),
each \(J_r\neq 0\).

The natural maps \(R/\mathfrak m^{r+1}\to R/\mathfrak m^r\) send \(J_{r+1}\)
onto \(J_r\). Choose a compatible system
\[
c_r\in J_r \qquad (r\ge k),
\]
with \(c_k\neq 0\). This defines an element
\[
c=(c_r)_{r\ge 1}\in \widehat R,
\]
and \(c\notin \widehat{\mathfrak m}^k\).

Fix \(n\). Since \(J_r\subseteq (I_n+\mathfrak m^r)/\mathfrak m^r\) for all
large \(r\), we get
\[
c\in I_n\widehat R+\widehat{\mathfrak m}^r
\]
for all \(r\gg 0\). Because ideals are closed in the complete Noetherian
local ring \(\widehat R\),
\[
\bigcap_{r\ge 1}(I_n\widehat R+\widehat{\mathfrak m}^r)=I_n\widehat R,
\]
hence \(c\in I_n\widehat R\) for every \(n\). Therefore
\[
c\in \bigcap_{n\ge 1} I_n\widehat R.
\]

Now \(c\neq 0\) since \(c\notin \widehat{\mathfrak m}^k\). Also,
\[
(c\widehat R)\cap R
\subseteq
\bigcap_{n\ge 1}(I_n\widehat R\cap R)
=
\bigcap_{n\ge 1}I_n
=0,
\]
because \(\widehat R\) is faithfully flat over \(R\).

Thus the ideal \(c\widehat R\) is disjoint from \(R\setminus\{0\}\). So there
exists a prime ideal
\[
P\in \operatorname{Spec}\widehat R
\]
with \(c\widehat R\subseteq P\) and \(P\cap R=0\). Since \(c\neq 0\), this
prime is nonzero. That contradicts the hypothesis.

So \(R\) is weakly quasi-complete.
\end{proof}

\begin{lemma}\label{lem:notweak_qc_crt}
If \(\widehat R\) has a nonzero prime ideal \(P\) with \(P\cap R=0\), then
\(R\) is not weakly quasi-complete.
\end{lemma}

\begin{proof}
Choose \(0\neq x\in P\). Pick \(k\) with
\(x\notin \widehat{\mathfrak m}^k\). Define
\[
I_n:=(P+\widehat{\mathfrak m}^n)\cap R.
\]
Then \(I_1\supseteq I_2\supseteq\cdots\), and
\[
\bigcap_{n\ge 1} I_n
=
\left(\bigcap_{n\ge 1}(P+\widehat{\mathfrak m}^n)\right)\cap R
=
P\cap R
=
0,
\]
because \(\widehat R/P\) is complete and separated.

Since \(R\) is dense in \(\widehat R\), for each \(n\) there exists
\(r_n\in R\) with
\[
x-r_n\in \widehat{\mathfrak m}^n.
\]
Then \(r_n\in I_n\). For \(n\) large enough, \(r_n\notin \mathfrak m^k\),
because otherwise
\[
x\in \mathfrak m^k+\widehat{\mathfrak m}^n
\subseteq \widehat{\mathfrak m}^k,
\]
contradiction.

So \(I_n\nsubseteq \mathfrak m^k\) for infinitely many \(n\), hence \(R\) is
not weakly quasi-complete.
\end{proof}

Thus we have the criterion:
\[
R \text{ is weakly quasi-complete}
\quad\Longleftrightarrow\quad
\text{every nonzero prime of } \widehat R
\text{ contracts nontrivially to } R.
\]

\subsection*{Step 1: \(A\) is weakly quasi-complete}

Let \(P\neq 0\) be a prime ideal of
\[
\widehat A = k[[X,Y]].
\]
Choose \(0\neq f\in P\). Since \(\operatorname{char}k=p\),
\[
f^p \in k^p[[X,Y]]\subseteq A.
\]
Also \(f^p\in P\), and \(f^p\neq 0\) because \(k[[X,Y]]\) is a domain. Hence
\[
0\neq f^p\in P\cap A.
\]
So every nonzero prime of \(\widehat A\) contracts nontrivially to \(A\). By
Lemma \ref{lem: weak_qc_crt}, \(A\) is weakly quasi-complete.

\subsection*{Step 2: a quotient of \(A\) is not weakly quasi-complete}

Define
\[
\phi(X):=\sum_{i\ge 1} u_i X^i \in Xk[[X]],
\qquad
w:=Y-\phi(X)\in k[[X,Y]].
\]
Then
\[
w^p = Y^p-\phi(X)^p
\]
because we are in characteristic \(p\). Set
\[
a:=w^p = Y^p-\phi(X)^p.
\]
Since
\[
\phi(X)^p=\sum_{i\ge 1} u_i^p X^{ip}\in k^p[[X]],
\]
we have \(a\in k^p[[X,Y]]\subseteq A\).

Now let
\[
B:=A/aA.
\]
Its completion is
\[
\widehat B \cong \widehat A/a\widehat A
= k[[X,Y]]/(w^p).
\]

Consider the ideal
\[
\mathfrak p:= (w)/(w^p)\subseteq \widehat B.
\]
This is a nonzero prime ideal, because
\[
\widehat B/\mathfrak p \cong k[[X,Y]]/(w)\cong k[[X]],
\]
a domain.

So to show \(B\) is not weakly quasi-complete, it remains to show that
\(\mathfrak p\cap B=0\).

\begin{claim*}
\[
A\cap w k[[X,Y]] = aA.
\]
\end{claim*}

\begin{proof}
For each \(n\ge 1\), set
\[
L_n:=k^p(u_1,\dots,u_n).
\]
Then
\[
A=\bigcup_{n\ge 1} L_n[[X,Y]],
\]
because every element of \(A\) involves only finitely many constants from
\(k\).

Let
\[
\varepsilon_n: L_n[[X,Y]] \longrightarrow k[[X]],
\qquad
F(X,Y)\mapsto F(X,\phi(X)).
\]
Its kernel is exactly \(L_n[[X,Y]]\cap w k[[X,Y]]\).

I claim
\[
\ker(\varepsilon_n)=(Y^p-\phi(X)^p)L_n[[X,Y]].
\]

Indeed, \(\phi(X)\notin L_n((X))\): after the \(n\)-th coefficient, the
coefficients are \(u_{n+1},u_{n+2},\dots\), and these do not lie in \(L_n\).
But
\[
\phi(X)^p=\sum_{i\ge 1}u_i^pX^{ip}\in k^p[[X]]\subseteq L_n((X)).
\]
Hence \(\phi(X)\) has minimal polynomial
\[
T^p-\phi(X)^p
\]
over \(L_n((X))\), so \(1,\phi,\dots,\phi^{p-1}\) are linearly independent
over \(L_n((X))\).

Now take \(F\in L_n[[X,Y]]\) with \(F(X,\phi(X))=0\). By Weierstrass division
by the monic polynomial \(Y^p-\phi(X)^p\) in the variable \(Y\),
\[
F =
Q(X,Y)(Y^p-\phi(X)^p)
+\sum_{i=0}^{p-1} b_i(X)Y^i,
\qquad
b_i(X)\in L_n[[X]].
\]
Evaluating at \(Y=\phi(X)\) gives
\[
\sum_{i=0}^{p-1} b_i(X)\phi(X)^i = 0.
\]
By the linear independence of \(1,\phi,\dots,\phi^{p-1}\) over \(L_n((X))\),
all \(b_i(X)=0\). So
\[
F\in (Y^p-\phi(X)^p)L_n[[X,Y]].
\]
This proves the claim about \(\ker(\varepsilon_n)\).

Therefore
\[
L_n[[X,Y]]\cap w k[[X,Y]]
=
(Y^p-\phi(X)^p)L_n[[X,Y]]
=
aL_n[[X,Y]].
\]
Taking the union over \(n\),
\[
A\cap w k[[X,Y]]
=
aA.
\]
The claim is proved.
\end{proof}

Now compute the contraction:
\[
\mathfrak p\cap B
=
\frac{(w k[[X,Y]]\cap A)+aA}{aA}
=
\frac{aA}{aA}
=
0.
\]
So \(\widehat B\) has a nonzero prime ideal contracting to \(0\) in \(B\).
By Lemma \ref{lem:notweak_qc_crt}, \(B\) is not weakly quasi-complete.

\subsection*{Conclusion}

We have exhibited a Noetherian local ring \(A\) such that:
\begin{itemize}
    \item \(A\) is weakly quasi-complete;
    \item the quotient \(B=A/aA\) is not weakly quasi-complete.
\end{itemize}

Therefore \(A\) cannot be quasi-complete, because a Noetherian local ring is
quasi-complete exactly when all of its homomorphic images are weakly
quasi-complete \cite{cahen2014open}.

So there does exist a weakly quasi-complete ring that is not quasi-complete.
The ring
\[
A = k^p[[X,Y]][k]
\quad\text{with}\quad
k=\mathbf F_p(u_1,u_2,\dots)
\]
is such an example.

\begin{remark*}
The above example is a valid counterexample to Anderson's open problem, found by the web version of GPT-5.4-Pro\footnote{\url{https://chatgpt.com/share/6a0de969-325c-83ea-acdc-0c705a267707}}, which is itself an agentic system.

This same example was originally constructed by Nagata \cite[Appendix A1, E3.1]{NagataLocalRings}, a master of counterexamples in commutative algebra. The appendix in Nagata’s book presents several pathological examples of Noetherian rings; this particular one serves as an intermediate step toward constructing another counterexample—namely, a local integral domain of altitude 1 whose derived normal ring is not a finite module. 

In contrast to the counterexample found by Rethlas, this example is a known pathological case, though it was originally introduced for a different purpose. The counterexample constructed by Rethlas, to the best of our knowledge, is new. Another important difference is that GPT-5.4-Pro produced an explicitly constructed example, while Rethlas gave a non-constructive example that relies heavily on the axiom of choice. Both proofs carry out the same reduction to construct a ring that is weakly quasi-complete but has a quotient that is not, and both use Farley's result to further reduce the weakly quasi-complete condition.
\end{remark*}

\section{Raw Output for the Algebraic Group Problem (Rethlas + GPT-5.5 xhigh)}

In this section, we present the raw output of Rethlas on rational conjugacy classes of 1-parameter subgroups in the theory of algebraic groups. Only formatting adjustments were made to display the Markdown file in LaTeX.
\begin{Verbatim}[commandchars=!??]
# lemma lem:split_unipotent_h1

## statement
Let $k$ be a field and let $U$ be a $k$-split smooth connected unipotent $k$-group. Then
$$
H^1(k,U)=1
$$
for nonabelian fppf cohomology. Consequently, if
$$
1\to U\to E\xrightarrow{\pi} Q\to 1
$$
is an exact sequence of smooth affine $k$-groups, then the map
$$
E(k)\to Q(k)
$$
is surjective.

## proof
Since $U$ is $k$-split, it has a composition series by smooth connected normal
$k$-subgroups whose successive quotients are $k$-isomorphic to $\mathbf G_ a$.
Induction on the length of such a series, using the exact sequence in nonabelian
fppf cohomology and $H^1(k,\mathbf G_ a)=0$, gives $H^1(k,U)=1$.

For the consequence, the fiber over any $q\in Q(k)$ is a left $U$-torsor over
$\operatorname{Spec} k$. Its cohomology class is trivial, so the fiber has a
$k$-point. Hence $E(k)\to Q(k)$ is surjective.

# lemma lem:split_unipotent_reduction

## statement
Let $G$ be a smooth connected affine $k$-group, let
$$
U=R_ {us,k}(G)
$$
be its maximal $k$-split smooth connected unipotent normal $k$-subgroup, and put
$$
\overline G=G/U.
$$
Assume that for every algebraic extension $K/k$ the natural map
$$
\mathbb X_ \ast(\overline G)/\overline G(k)\to
\mathbb X_ \ast(\overline G_ K)/\overline G(K)
$$
is injective. Then for every algebraic extension $K/k$ the natural map
$$
\mathbb X_ \ast(G)/G(k)\to \mathbb X_ \ast(G_ K)/G(K)
$$
is injective.

## proof
Let $\pi:G\to \overline G$ be the quotient map, and let
$\lambda,\mu\in \mathbb X_ \ast(G)$ become $G(K)$-conjugate. Then
$\overline\lambda=\pi\circ\lambda$ and
$\overline\mu=\pi\circ\mu$ become $\overline G(K)$-conjugate. By the assumed
injectivity for $\overline G$, there is $\overline g\in \overline G(k)$ such that
$$
\overline\mu=\operatorname{Int}(\overline g)\circ\overline\lambda.
$$
By Lemma `lem:split_unipotent_h1`, $G(k)\to\overline G(k)$ is surjective, so
choose $g\in G(k)$ lifting $\overline g$. Replacing $\lambda$ by
$\operatorname{Int}(g)\circ\lambda$, we reduce to the case
$$
\pi\circ\lambda=\pi\circ\mu=:\nu:\mathbf G_ m\to\overline G.
$$

Form the pullback $k$-group
$$
E:=\mathbf G_ m\times_ {\nu,\overline G,\pi}G.
$$
Then $E$ is smooth connected affine and fits into an exact sequence
$$
1\to U\to E\xrightarrow{q}\mathbf G_ m\to 1.
$$
The two cocharacters $\lambda$ and $\mu$ are the same thing as two
$k$-homomorphic sections
$$
s_ \lambda(t)=(t,\lambda(t)),\qquad s_ \mu(t)=(t,\mu(t))
$$
of $q$.

We use the following cited result.

Complete cited statement: Theorem 4.2.9 in *Structure and classification of pseudo-
reductive groups* states: "Any two maximal split $k$-tori in a smooth connected affine
$k$-group $G$ are conjugate under $G(k)$."

paper_id: Structure and classification of pseudo-reductive groups
theorem_id: Theorem 4.2.9
arXiv id:

The images $s_ \lambda(\mathbf G_ m)$ and $s_ \mu(\mathbf G_ m)$ are split
one-dimensional $k$-tori in $E$. They are maximal split $k$-tori: a split torus
in $E$ maps injectively to $\mathbf G_ m$, because the kernel $U$ is unipotent
and contains no nontrivial torus. By the cited theorem, there is $e\in E(k)$ such that
$$
e\,s_ \lambda(\mathbf G_ m)\,e^{-1}=s_ \mu(\mathbf G_ m).
$$
Let $a=q(e)\in k^\times=\mathbf G_ m(k)$. Since $s_ \mu(a)$ centralizes
$s_ \mu(\mathbf G_ m)$, the element
$$
u:=s_ \mu(a)^{-1}e
$$
lies in $U(k)$ and still carries $s_ \lambda(\mathbf G_ m)$ onto
$s_ \mu(\mathbf G_ m)$. Moreover $q\circ\operatorname{Int}(u)\circ s_\lambda=q\circ 
s_ \lambda$.
The restriction $q|_ {s_ \mu(\mathbf G_ m)}$ is an isomorphism onto $\mathbf G_ m$,
so the only section of $q$ with image $s_ \mu(\mathbf G_ m)$ is $s_ \mu$. Thus
$$
s_ \mu=\operatorname{Int}(u)\circ s_ \lambda,
$$
and hence $\mu=\operatorname{Int}(u)\circ\lambda$. Therefore $\lambda$ and
$\mu$ are $G(k)$-conjugate.

# lemma lem:geometric_normalizer_conjugacy

## statement
Let $F$ be an algebraically closed field, let $H$ be a smooth connected affine
$F$-group, and let $T\subset H$ be a maximal torus. If
$\lambda,\mu\in X_ \ast(T)$ are $H(F)$-conjugate, then they are conjugate by an
element of $N_ H(T)(F)$.

## proof
Choose $g\in H(F)$ such that $\mu=\operatorname{Int}(g)\circ\lambda$. Let
$$
C=Z_ H(\mu(\mathbf G_ m))^\circ.
$$
The standard torus-centralizer theorem says that the centralizer of a torus in a
smooth affine group is smooth; hence $C$ is a smooth connected affine $F$-group.
Both $T$ and $gTg^{-1}$ contain $\mu(\mathbf G_ m)$, so they are contained in
$C$. They are maximal tori of $C$, since any larger torus in $C$ would be a
larger torus in $H$.

Over the algebraically closed field $F$, maximal tori are maximal split tori. Applying
Theorem 4.2.9 from *Structure and classification of pseudo-reductive groups* to the
smooth connected affine group $C$, choose $c\in C(F)$ with
$$
c\,gTg^{-1}c^{-1}=T.
$$
Then $n:=cg\in N_ H(T)(F)$, and because $c$ centralizes $\mu(\mathbf G_ m)$,
$$
\operatorname{Int}(n)\circ\lambda
=\operatorname{Int}(c)\circ\mu
=\mu.
$$
Thus $\lambda$ and $\mu$ are $N_ H(T)(F)$-conjugate.

# lemma lem:wound_relative_dominant_representatives

## statement
Let $H$ be a smooth connected affine $k$-group such that $R_ {u,k}(H)$ is
$k$-wound. Let $S\subset H$ be a maximal split $k$-torus, and let $P$ be a
minimal pseudo-parabolic $k$-subgroup containing $S$. Let
$$
{}^k\Phi=\Phi(H/R_ {u,k}(H),S)
$$
be the relative root system with positive system ${}^k\Phi^+=\Phi(P/R_ {u,k}(H),S)$,
and set
$$
C=\{\nu\in X_ \ast(S)\mid \langle a,\nu\rangle\ge 0
\text{ for all }a\in{}^k\Phi^+\}.
$$
Then:

1. every element of $X_ \ast(S)$ is $N_ H(S)(k)$-conjugate to a unique element of $C$;
2. if $\lambda,\mu\in X_ \ast(S)$ become $H(K)$-conjugate over an algebraic
extension $K/k$, then they have the same unique representative in $C$. In
particular, $\lambda$ and $\mu$ are $N_ H(S)(k)$-conjugate.

## proof
We use the following cited results.

Complete cited statement: Proposition 5.3.1 in *Structure and classification of pseudo-
reductive groups* states that if $S$ is a maximal split $k$-torus 
in a smooth connected affine $k$-group $G$, then
$$
W(G,S):=N_ G(S)/Z_ G(S)
$$
is a constant finite etale $k$-group and the inclusion
$$
N_ G(S)(k)/Z_ G(S)(k)\hookrightarrow W(G,S)(k)
$$
is an equality.

paper_id: Structure and classification of pseudo-reductive groups
theorem_id: Proposition 5.3.1
arXiv id:

Complete cited statement: Theorem 5.3.2(i) in *Structure and classification of pseudo-
reductive groups* states that for a smooth connected affine $k$-group $G$, a maximal split 
$k$-torus $S$, and a minimal pseudo-parabolic $k$-subgroup $P$ containing $S$, the set
$$
{}^k\Phi=\Phi(G/R_ {u,k}(G),S)
$$
is a root system in its $\mathbf Q$-span in $X(S)_ \mathbf Q$, the subset
$$
\Phi(P/R_ {u,k}(G),S)
$$
is a positive system of roots, and the natural map
$$
N_ G(S)(k)/Z_ G(S)(k)\to W({}^k\Phi)
$$
is an isomorphism.

paper_id: Structure and classification of pseudo-reductive groups
theorem_id: Theorem 5.3.2(i)
arXiv id:

Complete cited statement: Theorem 5.3.2(iv) in *Structure and classification of pseudo-
reductive groups* states that if $R_ {u,k}(G)$ is $k$-wound, then the root system 
$\Phi(G,S)$ consists of the nontrivial $S$-weights on $\operatorname{Lie}(G)$.

paper_id: Structure and classification of pseudo-reductive groups
theorem_id: Theorem 5.3.2(iv)
arXiv id:

By Theorem 5.3.2(i), $N_ H(S)(k)/Z_ H(S)(k)$ is identified with the Weyl group
of the root system ${}^k\Phi$. The usual root-system argument shows that each Weyl
orbit on $X_ \ast(S)$ has a unique representative in the closed dominant chamber
$C$. Proposition 5.3.1 realizes these Weyl elements by $k$-points of $N_ H(S)$.
This proves part (1).

For part (2), let $\lambda^+,\mu^+\in C$ be the unique representatives of
$\lambda$ and $\mu$ from part (1). Since the conjugations from $\lambda$ to
$\lambda^+$ and from $\mu$ to $\mu^+$ are already defined over $k$, the
cocharacters $\lambda^+$ and $\mu^+$ are still $H(K)$-conjugate. Replacing
$\lambda,\mu$ by $\lambda^+,\mu^+$, we may assume from now on that
$$
\lambda,\mu\in C.
$$

Let $k_ s$ be a separable closure of $k$. Choose a maximal $k_ s$-split torus
$$
T\subset H_ {k_ s}
$$
containing $S_ {k_ s}$. Choose a minimal pseudo-parabolic $k_ s$-subgroup
$$
B\subset P_ {k_ s}
$$
containing $T$. Let
$$
\Phi_ {\mathrm{abs}}=\Phi(H_ {k_ s}/R_ {u,k_ s}(H_ {k_ s}),T)
$$
be the corresponding absolute root system, with positive system
$\Phi_ {\mathrm{abs}}^+=\Phi(B/R_ {u,k_ s}(H_ {k_ s}),T)$.

For $\nu\in X_ \ast(S)\subset X_ \ast(T)$ and $\alpha\in\Phi_ {\mathrm{abs}}$,
$$
\langle \alpha,\nu\rangle=\langle \alpha|_ S,\nu\rangle.
$$
By Theorem 5.3.2(iv), because the $k$-wound property of $R_ {u,k}(H)$ is preserved
under separable extension, the relative and absolute roots are exactly the nontrivial
weights of $S$ and $T$ on the corresponding Lie algebras. Since $B\subset P_ {k_ s}$,
every $\alpha\in\Phi_ {\mathrm{abs}}^+$ restricts either to $0$ or to an element of
${}^k\Phi^+$, and every element of ${}^k\Phi^+$ arises as the nonzero restriction
of some element of $\Phi_ {\mathrm{abs}}^+$. Hence, for cocharacters in $X_ \ast(S)$,
dominance for ${}^k\Phi^+$ is equivalent to dominance for
$\Phi_ {\mathrm{abs}}^+$. Thus $\lambda$ and $\mu$ are also
$\Phi_ {\mathrm{abs}}^+$-dominant.

Let $\Omega$ be an algebraic closure containing both $K$ and $k_ s$. Since
$\lambda$ and $\mu$ are $H(K)$-conjugate, they are $H(\Omega)$-conjugate.
The torus $T_ \Omega$ is a maximal torus of $H_ \Omega$. By Lemma
`lem:geometric_normalizer_conjugacy`, $\lambda$ and $\mu$ are conjugate by an
element of $N_ {H_ \Omega}(T_ \Omega)(\Omega)$.

By Proposition 5.3.1 applied over $k_ s$ to the maximal split torus $T$, the finite
etale group $N_ {H_ {k_ s}}(T)/Z_ {H_ {k_ s}}(T)$ is constant; therefore purely inseparable
extension from $k_ s$ to $\Omega$ adds no new Weyl elements. By Theorem 5.3.2(i),
this Weyl group is $W(\Phi_ {\mathrm{abs}})$. Thus $\lambda$ and $\mu$ lie in the
same $W(\Phi_ {\mathrm{abs}})$-orbit.

Every Weyl-group orbit in a root system has a unique element in the closed dominant
chamber. Since both $\lambda$ and $\mu$ are $\Phi_ {\mathrm{abs}}^+$-dominant, we get
$$
\lambda=\mu.
$$
Undoing the initial $N_ H(S)(k)$-conjugations to dominant representatives proves that
$\lambda$ and $\mu$ have the same unique representative in $C$, and hence are
$N_ H(S)(k)$-conjugate.

# lemma lem:wound_unipotent_case

## statement
Let $H$ be a smooth connected affine $k$-group such that $R_ {u,k}(H)$ is
$k$-wound. For every algebraic extension $K/k$, the natural map
$$
\mathbb X_ \ast(H)/H(k)\to\mathbb X_ \ast(H_ K)/H(K)
$$
is injective.

## proof
Let $\lambda,\mu\in\mathbb X_ \ast(H)$ become $H(K)$-conjugate. We use again
Theorem 4.2.9 from *Structure and classification of pseudo-reductive groups*: any
two maximal split $k$-tori in a smooth connected affine $k$-group are conjugate
under $k$-rational points.

Choose a maximal split $k$-torus $S\subset H$. Since the image of any
$k$-cocharacter is a split $k$-torus, Theorem 4.2.9 lets us conjugate $\lambda$
and $\mu$, separately by elements of $H(k)$, so that both lie in $X_ \ast(S)$.
These conjugated cocharacters are still $H(K)$-conjugate.

Choose a minimal pseudo-parabolic $k$-subgroup $P$ containing $S$. Applying
Lemma `lem:wound_relative_dominant_representatives` to $H,S,P$, the two
cocharacters in $X_ \ast(S)$ are $N_ H(S)(k)$-conjugate. Hence the original
$\lambda$ and $\mu$ are $H(k)$-conjugate.

# theorem thm:main

## statement
Let $G$ be a (not necessarily reductive) smooth connected affine group over a  (not 
necessarily perfect) field $k$. Let $K$ be an algebraic field extension of $k$. Let 
$\mathbb{X}_ \ast(G):=\mathrm{Hom}_ {k\text{-groups}}(\mathbb{G}_ m,G)$ $k$-homomorphisms. 
Let $G_ K$ be the base change of $G$ to $K$. Prove (or disprove with a counterexample) that 
$\mathbb{X}_ \ast(G)/G(k)\rightarrow\mathbb{X}_ \ast(G_ K)/G(K)$ is injective. (You can
take for granted that this statement is true if $k$ is a perfect field.)

## proof
We prove that the displayed map is injective; hence there is no counterexample.

Let
$$
U=R_ {us,k}(G)
$$
be the maximal $k$-split smooth connected unipotent normal $k$-subgroup of $G$,
and set $\overline G=G/U$.

We use the following cited result.

Complete cited statement: Corollary 3.12 in *The structure of solvable groups over
fields* states that for any smooth connected affine $k$-group $G$,
$$
R_ {us,k}(G)=R_ {u,k}(G)_ {\mathrm{split}}.
$$
In particular,
$$
R_ {u,k}(G)/R_ {us,k}(G)
$$
is $k$-wound.

paper_id: The structure of solvable groups over fields
theorem_id: Corollary 3.12
arXiv id:

Applying this result to $G$, the unipotent radical of $\overline G$ is
$$
R_ {u,k}(\overline G)=R_ {u,k}(G)/R_ {us,k}(G),
$$
which is $k$-wound. Therefore Lemma `lem:wound_unipotent_case` proves injectivity
for $\overline G$.

Finally Lemma `lem:split_unipotent_reduction` lifts injectivity from $\overline G$
to $G$. Thus, if two $k$-homomorphisms
$$
\mathbf G_ m\to G
$$
become conjugate by an element of $G(K)$, then they were already conjugate by an
element of $G(k)$. Equivalently,
$$
\mathbb X_ \ast(G)/G(k)\to \mathbb X_ \ast(G_ K)/G(K)
$$
is injective.
\end{Verbatim}

\section{Raw Output for the Algebraic Group Problem (GPT-5.5 Pro)}\label{app:gpt5.5pro-output}
We prove that the map
\[
\mathbb X_*(G)/G(k)\longrightarrow \mathbb X_*(G_K)/G(K)
\]
is injective.

Let
\[
\lambda,\mu\in \mathbb X_*(G)
=
\operatorname{Hom}_{k\text{-groups}}(\mathbb G_m,G),
\]
and suppose that their images in
\(\mathbb X_*(G_K)/G(K)\) are equal. Thus there exists
\(g\in G(K)\) such that
\[
\mu_K=\operatorname{Int}(g)\circ \lambda_K.
\]

Let \(k^{\mathrm{perf}}\) be the perfect closure of \(k\) inside an
algebraic closure of \(k\). Since \(K/k\) is algebraic, the compositum
\[
\Omega:=Kk^{\mathrm{perf}}
\]
is algebraic over \(k^{\mathrm{perf}}\). Over \(\Omega\), the
cocharacters \(\lambda\) and \(\mu\) are conjugate. By the assumed result
over perfect fields, applied to the perfect field \(k^{\mathrm{perf}}\),
there exists
\[
g_0\in G(k^{\mathrm{perf}})
\]
such that
\[
\mu_{k^{\mathrm{perf}}}
=
\operatorname{Int}(g_0)\circ \lambda_{k^{\mathrm{perf}}}.
\]

Define the transporter
\[
T:=\operatorname{Transp}_G(\lambda,\mu)
\]
by
\[
T(R)
=
\left\{
x\in G(R):
\operatorname{Int}(x)\circ \lambda_R=\mu_R
\right\}
\]
for every \(k\)-algebra \(R\). This is a closed \(k\)-subscheme of \(G\).
Moreover, \(T\) is a right torsor under
\[
C:=C_G(\lambda)=C_G(\operatorname{Im}\lambda)
\]
via right multiplication. The group \(C\) is smooth, since it is the
centralizer of a torus in a smooth affine group.

The element \(g_0\) gives
\[
T(k^{\mathrm{perf}})\neq \varnothing.
\]
Since \(T\) is of finite type over \(k\), this
\(k^{\mathrm{perf}}\)-point is already defined over some finite purely
inseparable extension \(E/k\). Hence
\[
T(E)\neq \varnothing.
\]

Thus the \(C\)-torsor \(T\) becomes trivial over \(E\). Since \(C\) is
smooth, fppf \(C\)-torsors and \'{e}tale \(C\)-torsors agree. A finite
purely inseparable morphism
\[
\operatorname{Spec}E\longrightarrow \operatorname{Spec}k
\]
is a universal homeomorphism, and universal homeomorphisms induce
equivalences of \'{e}tale topoi. Therefore the pullback map
\[
H^1_{\acute{e}t}(k,C)
\longrightarrow
H^1_{\acute{e}t}(E,C_E)
\]
is bijective.

Since the class of \(T\) maps to the trivial class over \(E\), it was
already trivial over \(k\). Hence
\[
T(k)\neq \varnothing.
\]
Choose \(h\in T(k)\). By definition of \(T\), we have
\[
\mu=\operatorname{Int}(h)\circ \lambda.
\]
Therefore \(\lambda\) and \(\mu\) are \(G(k)\)-conjugate.

Hence the map
\[
\mathbb X_*(G)/G(k)
\longrightarrow
\mathbb X_*(G_K)/G(K)
\]
is injective.
\begin{remark*}
	The proof above, produced by giving the general-case problem to the GPT-5.5-Pro web version\footnote{\url{https://chatgpt.com/share/69f70c92-f834-83a5-8bbf-34a6dc254c39}}, which is itself an agentic system, is completely wrong. The strategy, namely the naive reduction to the perfect-field case, does not actually work, and the claimed bijectivity of the map 
    \[
	H^1_{\mathrm{\acute{e}t}}(k,C)\rightarrow H^1_{\mathrm{\acute{e}t}}(E,C_E)
	\]
    
    is wrong. We also tested GPT-5.5-Pro through its webpage interface by first providing the three special cases and then the general case; it again produced an incorrect proof. The formatted raw output\footnote{\url{https://chatgpt.com/share/6a080c76-e1e4-83ea-a177-b3d381b318fc}} is available at the \href{https://srliu3264.github.io/rethlas_results_math/gpt_5.5_pro_rational_conjugacy_classes_of_1_parameter_subgroups/}{GPT-5.5-Pro (Web) Results Homepage}.
\end{remark*}

\section{Correspondence between Natural Language Proof and Formalization}
This appendix details the correspondence between the natural-language proof presented in Appendix~A and its Lean~4 formalization. All file paths are relative to the project root \texttt{Anderson/}. We first give the correspondence for the main proof components, then describe how the external results cited in the proof are formalized and where the relevant reference materials reside.

\subsection*{Main theorem and definitions}

\begin{center}
\small
\begin{tabular}{p{5.5cm} l l}
\toprule
\textbf{Informal component} & \textbf{Lean declaration} & \textbf{File} \\
\midrule
Def.\ of quasi-complete & \texttt{IsQuasiComplete} & \texttt{Basic.lean} \\
Def.\ of weakly quasi-complete & \texttt{IsWeaklyQuasiComplete} & \texttt{Basic.lean} \\
Def.\ of analytically irreducible & \texttt{IsAnalyticallyIrreducible} & \texttt{Basic.lean} \\
QC $\Rightarrow$ WQC & \texttt{IsQuasiComplete.isWeaklyQuasiComplete} & \texttt{Basic.lean} \\
Main theorem (\Cref{thm:main-appendix}) & \texttt{main\_theorem} & \texttt{Main.lean} \\
\bottomrule
\end{tabular}
\end{center}

\subsection*{Construction of $T$ (\Cref{lem:complete_domain_choice-appendix})}

\begin{center}
\footnotesize
\begin{tabular}{p{5.5cm} l l}
\toprule
\textbf{Informal component} & \textbf{Lean declaration} & \textbf{File} \\
\midrule
$T = \mathbb{C}[[x,y,z]]/(x^2-yz)$ & \texttt{T} & \texttt{CompleteDomain/Domain.lean} \\
$T$ is a domain & \texttt{T\_isDomain} & \texttt{CompleteDomain/Domain.lean} \\
\multicolumn{3}{l}{$T$ is a complete local ring, decomposed as:} \\
\quad local ring & \texttt{T\_isLocalRing} & \texttt{CompleteDomain/LocalRing.lean} \\
\quad Noetherian & \texttt{T\_isNoetherianRing} & \texttt{CompleteDomain/LocalRing.lean} \\
\quad $M$-adically complete & \texttt{T\_isAdicComplete} & \texttt{CompleteDomain/LocalRing.lean} \\
\multicolumn{3}{l}{$T$ is Cohen--Macaulay of dimension $2$, decomposed as:} \\
\quad $\dim T = 2$ & \texttt{T\_ringKrullDim} & \texttt{CompleteDomain/CompleteDomain.lean} \\
\quad $\operatorname{depth} T \ge 2$ & \texttt{jensen\_T\_depth\_ge\_two} & \texttt{Jensen/Jensen.lean} \\
$|T| = |T/M| = |\mathbb{C}|$ & \texttt{jensen\_T\_card\_eq\_residue\_card} & \texttt{Jensen/Jensen.lean} \\
$Q = (x,y)T$ defined & \texttt{Q} & \texttt{CompleteDomain/CompleteDomain.lean} \\
$Q$ is prime, height 1 & \texttt{Q\_isPrime}, \texttt{Q\_height\_one} & \texttt{CompleteDomain/CompleteDomain.lean} \\
$Q$ is not principal & \texttt{Q\_not\_isPrincipal} & \texttt{CompleteDomain/CompleteDomain.lean} \\
\bottomrule
\end{tabular}
\end{center}

\subsection*{Construction of $A$ (\Cref{lem:jensen_special_case-appendix})}

\begin{center}
\small
\begin{tabular}{p{5.5cm} l l}
\toprule
\textbf{Informal component} & \textbf{Lean declaration} & \textbf{File} \\
\midrule
Jensen Cor.~2.4 applied & \texttt{jensen\_special\_case} & \texttt{Jensen/Jensen.lean} \\
Trivial generic formal fiber & \texttt{HasTrivialGenericFormalFiber} & \texttt{Jensen/Defs.lean} \\
\bottomrule
\end{tabular}
\end{center}

\subsection*{Verification (conditions (A) and (B))}

\begin{center}
\small
\begin{tabular}{p{5.5cm} l l}
\toprule
\textbf{Informal component} & \textbf{Lean declaration} & \textbf{File} \\
\midrule
$A$ is WQC (condition (A)) & \texttt{a\_isWeaklyQuasiComplete} & \texttt{Main.lean} \\
$Q \ne (0)$ & \texttt{Q\_ne\_bot} & \texttt{Main.lean} \\
$q = Q \cap A$, $\operatorname{ht}(q) = 1$ & \texttt{contraction\_height\_one} & \texttt{Main.lean} \\
$q = (a)$ for prime $a$ (UFD) & \texttt{ufd\_height\_one\_principal} & \texttt{Main.lean} \\
$\dim(A/aA) = 1$ & \texttt{quotient\_prime\_dim\_one} & \texttt{Main.lean} \\
$A/aA$ not analytically irred. & \texttt{quotient\_not\_analytically\_irreducible} & \texttt{Main.lean} \\
$\exists$ bad quotient (condition (B)) & \texttt{exists\_prime\_bad\_quotient} & \texttt{Main.lean} \\
\bottomrule
\end{tabular}
\end{center}

\subsection*{Cited results from the literature}

The informal proof relies on several external results. We describe how each is formalized and where the corresponding reference material is stored (in the \texttt{references/} directory of the project).

\paragraph{Anderson (2014), \emph{Quasi-complete semilocal rings and modules}~\cite{anderson2014quasi}.}

Three results from this paper are used in the proof:
\begin{itemize}
\item \emph{Corollary~2.2} (one-dimensional local domain: WQC $\Leftrightarrow$ analytically irreducible) is formalized as \texttt{dim1\_wqc\_iff\_analyticallyIrreducible} in \texttt{QuasiCompleteRing/QuasiCompleteRing.lean}.
\item \emph{Theorem~5, Item~3} (QC $\Leftrightarrow$ every proper quotient is WQC) is formalized as \texttt{isQuasiComplete\_iff\_quotients\_wqc} in \texttt{QuasiCompleteRing/QuasiCompleteRing.lean}.
\item \emph{Theorem~3} (complete $\Rightarrow$ quasi-complete), originally due to Chevalley~\cite{chevalley1943theory}, is formalized as \texttt{anderson\_complete\_isQuasiComplete} in \texttt{QuasiCompleteRing/Complete.lean}, following Anderson's proof. This result appears in the background discussion of Appendix~A but is not directly required in the main proof chain.
\end{itemize}
Reference material: \texttt{references/anderson\_2014.md}.

\paragraph{Farley (2016), \emph{Quasi-completeness and localizations of polynomial domains}~\cite{farley2016quasi}.}

Proposition~1 (Noetherian local domain is WQC iff every nonzero prime of the completion meets the ring nontrivially) is formalized as \texttt{isWeaklyQuasiComplete\_iff\_primes\_meet} in \texttt{QuasiCompleteRing/QuasiCompleteRing.lean}. This characterization is the key tool for establishing condition~(A).\\
Reference material: \texttt{references/farley\_2016.md}.

\paragraph{Jensen (2006), \emph{Completions of UFDs with Semi-Local Formal Fibers}~\cite{jensen2006completions}.}

Corollary~2.4 (existence of a local UFD with prescribed completion and prescribed generic formal fiber) is formalized as \texttt{jensen\_special\_case} in \texttt{Jensen.lean}, which applies the general construction to the specific ring~$T$.

The underlying construction (Theorem~2.2 of Jensen, building on Heitmann~\cite{heitmann1993characterization} and Loepp~\cite{loepp1997constructing}) is a transfinite recursive procedure that builds a chain of N-subrings converging to the desired UFD. This is the most technically involved part of the formalization, spanning the entire \texttt{Jensen/} subdirectory ($\sim$15{,}000 lines). The key modules are:
\begin{itemize}
\item \texttt{Jensen/Construction/}: the main transfinite recursion and its well-foundedness argument.
\item \texttt{Jensen/CloseUp/}: the close-up procedure ensuring that finitely generated ideals remain closed under extension.
\item \texttt{Jensen/KrullDomain/}: UFD verification for intersections of subrings via Kaplansky's criterion.
\item \texttt{Jensen/TransfiniteUnion.lean}: verification that unions at limit ordinals preserve N-subring properties.
\item \texttt{Jensen/Adjoin/}: adjoining transcendental elements while maintaining algebraic invariants.
\item \texttt{Jensen/NSubring.lean}: definition of N-subrings and A-extensions.
\item \texttt{Jensen/Application.lean}: application of the general construction to the specific ring~$T$.
\end{itemize}
Reference material: \texttt{references/jensen\_2006.md}, \texttt{references/heitmann\_1993.md}, \texttt{references/loepp\_1997.md}.

\paragraph{Chevalley (1943), \emph{On the theory of local rings}~\cite{chevalley1943theory}.}

Chevalley's result that completeness implies quasi-completeness is formalized as \texttt{anderson\_complete\_isQuasiComplete} in \texttt{QuasiCompleteRing/Complete.lean}, following the proof given in Anderson's generalization (Theorem~3 of~\cite{anderson2014quasi}).\\
Reference material: \texttt{references/chevalley\_1943.md}.

\section{Human-Provided Proof Blueprint}\label{appendix:human-blueprint}

This appendix reproduces the proof blueprint provided to Archon by a human supervisor during the ablation experiment described in \Cref{subsubsec:ablation-human}. This blueprint was \emph{not} used in the main experiment. The document was supplied as a Markdown file and is presented here with only formatting adjustments. It targets the proof of \texttt{UniqueFactorizationMonoid S\_sub} inside the Jensen construction module, proposing a Krull domain approach as an alternative to the Kaplansky criterion route that Archon had been pursuing.

\subsection*{Mathematical setting}

Let $(T, M)$ be a complete Noetherian local domain. Let $R$ be an N-subring of $T$ (in particular a local UFD with $R \subseteq T$). Let $y_1, y_2 \in R$ be coprime in the sense that no prime element of $R$ divides both $y_1$ and $y_2$. Let $x_1, x_2 \in T$ be transcendental over $R$ with $c = x_1 y_1 + x_2 y_2$ where $c \in R$.

Define:
\begin{itemize}
    \item $A_1 = R[x_1, y_2^{-1}]$,\quad $A_2 = R[x_2, y_1^{-1}]$
    \item $\tilde{R} = A_1 \cap A_2$
    \item $S_{\mathrm{sub}} = (\tilde{R} \cap (T \setminus M))^{-1}\tilde{R}$
\end{itemize}

\noindent\textbf{Goal:} $S_{\mathrm{sub}}$ is a UFD.

\subsection*{Formalization suggestions}

One should define the structure \texttt{IsKrullRing} for a ring to mean that it (its image) is the intersection of all valuation subrings of its fraction field that contain it. This is equivalent to: given that $K$ is the fraction field of $R$, the image of $R$ is the intersection of all valuation subrings of $K$ containing it. The intersection of two Krull rings with the same fraction field is again a Krull ring.

Let $d$ be an element of $R$, and $K$ the fraction field of $R$. One should define a property of a valuation subring of $K$ as being ``generated by $d$'' if it contains $R$ and its maximal ideal is generated by $d$. In the definition of the structure ``valuation subring in $K$ generated by an element'' (\texttt{ValuationSubring.IsPrincipalGen}), this is a property of a valuation subring of the fraction field of $R$, requiring that it is a DVR and its maximal ideal is generated by (the image of) some prime element $d$ of $R$. A valuation subring generated by a prime element $d$ is naturally a valuation subring generated by an element.

Furthermore, given a prime element $d$ in a Krull domain $R$, one needs to concretely construct a valuation subring of $K$ that is precisely the localization of $R$ at the prime ideal generated by $d$. One should first construct this subring, prove (non-trivially, following the supporting lemma at the end of this appendix) that it is a valuation subring, and then upgrade it to a valuation subring.

Using these definitions, one can better formalize the following proof of the proposition.

\subsection*{Proof}

\paragraph{Step 1: Reduction to $\tilde{R}$.} Since the localization of a UFD at any multiplicative set is again a UFD, it suffices to prove that $\tilde{R}$ is a UFD.

\paragraph{Step 2: $A_1$ and $A_2$ are UFDs and Krull domains.} Since $R$ is a UFD and $x_1, x_2$ are transcendental over $R$, the polynomial rings $R[x_1]$ and $R[x_2]$ are UFDs. Localizing at $y_2$ and $y_1$ respectively preserves the UFD property, so $A_1$ and $A_2$ are UFDs. Every UFD is a Krull domain: the localization at the prime ideal generated by any prime element $p$ is a DVR, the UFD is the intersection of all such DVRs over its fraction field, and any non-zero element has non-zero valuation in only finitely many of these DVRs.

\paragraph{Step 3: $\tilde{R}$ is a Krull domain.} A finite intersection of Krull domains with the same fraction field is again a Krull domain. The set of essential DVRs for $\tilde{R}$ is obtained by taking the union of the DVR families of $A_1$ and $A_2$. The finiteness condition (any element has non-zero valuation in only finitely many DVRs) is inherited.

\paragraph{Step 4: $y_1 y_2$ is a product of prime elements in $\tilde{R}$.} Let $p$ be a prime factor of $y_1$ in $R$. In $A_1 = R[x_1, y_2^{-1}]$, $p$ remains prime because adjoining $x_1$ and inverting $y_2$ (coprime to $y_1$) does not affect $p$. In $A_2 = R[x_2, y_1^{-1}]$, $p$ becomes a unit since $p \mid y_1$ and $y_1$ is inverted. In the Krull domain $\tilde{R}$, the valuation profile of $p$ is therefore $1$ at exactly one DVR (from the $A_1$ side) and $0$ everywhere else. An element with this profile cannot be decomposed further, so $p$ is prime in $\tilde{R}$. The same argument applies symmetrically to prime factors of $y_2$.

\paragraph{Step 5: $\tilde{R}[(y_1 y_2)^{-1}]$ is a UFD.} One computes:
\[
\tilde{R}[(y_1 y_2)^{-1}] = A_1[y_1^{-1}] \cap A_2[y_2^{-1}] = R[x_1, y_1^{-1}, y_2^{-1}] \cap R[x_2, y_1^{-1}, y_2^{-1}].
\]
The relation $c = x_1 y_1 + x_2 y_2$ implies $x_2 = (c - x_1 y_1)/y_2$, so $R[x_1, y_1^{-1}, y_2^{-1}] = R[x_2, y_1^{-1}, y_2^{-1}]$ and hence
\[
\tilde{R}[(y_1 y_2)^{-1}] = R[x_1, (y_1 y_2)^{-1}],
\]
which is a localization of the polynomial UFD $R[x_1]$ and therefore a UFD.

\paragraph{Step 6: Applying Nagata's lemma.} We are now left with a classical setup (a variant of Nagata's Theorem): $\tilde{R}$ is a Krull domain, $y = y_1 y_2$ is a product of prime elements in $\tilde{R}$, and $\tilde{R}[y^{-1}]$ is a UFD. We must prove $\tilde{R}$ is a UFD. To do so, we need only prove the existence and uniqueness of prime factorizations for any element in $\tilde{R}$.

\emph{Existence of factorization.} Take an arbitrary non-zero, non-unit element $a \in \tilde{R}$. Because $\tilde{R}[y^{-1}]$ is a UFD, we can factor $a$ in the localized ring. To pull this factorization back to $\tilde{R}$, we adjust the powers of the prime factors of $y$. Since $\tilde{R}$ is a Krull domain, the valuation of $a$ at the DVRs corresponding to the prime factors of $y$ is strictly finite (this prevents the pathological case of extracting infinite powers of $y$). We can perfectly factor out the exact finite powers of these $y$-primes from $a$, leaving an element $a'$ whose prime factorization in $\tilde{R}[y^{-1}]$ consists entirely of primes coprime to $y$. We multiply by appropriate powers of $y$ to clear denominators so that this factorization resides entirely in $\tilde{R}$.

\emph{Primality of the pulled-back factors.} We must verify that these pulled-back factors (let us call an arbitrary one $q$) remain prime in the base ring $\tilde{R}$. Suppose for contradiction that $q$ is not prime in $\tilde{R}$. Since $q$ maps to a prime in $\tilde{R}[y^{-1}]$, its valuation profile in the localized ring has exactly one non-zero entry (which is $1$). If $q$ decomposes in $\tilde{R}$, its valuation profile over the DVRs of $\tilde{R}$ must split, meaning $q$ would either have valuation ${>}\,0$ at two or more DVRs, or valuation ${>}\,1$ at a single DVR. However, the set of DVRs of the localization $\tilde{R}[y^{-1}]$ is exactly the set of DVRs of $\tilde{R}$ minus those corresponding to the prime factors of $y$. Because we specifically adjusted $q$ such that its valuation at the $y$-DVRs is $0$, the valuation profile of $q$ in $\tilde{R}$ is identical to its valuation profile in $\tilde{R}[y^{-1}]$: exactly $1$ at a single DVR, and $0$ everywhere else. An element with $v(q) = 1$ at a single DVR and $0$ elsewhere cannot be decomposed further. Therefore, $q$ is prime in $\tilde{R}$.

\paragraph{Step 7: Conclusion.} Since any element in $\tilde{R}$ can be factored into prime elements, $\tilde{R}$ is a UFD. Finally, because $S_{\mathrm{sub}}$ is a localization of $\tilde{R}$ at the multiplicative set $\tilde{R} \cap (T \setminus M)$, and the localization of a UFD is always a UFD, we conclude that $S_{\mathrm{sub}}$ is a UFD.

\subsection*{Supporting lemma: localization at a prime element in a Krull domain}

\noindent\textbf{Lemma.} \textit{Let $A$ be a Krull domain and let $t \in A$ be a prime element. Then the localization $A_{(t)}$ at the prime ideal $(t)$ is a DVR.}

\begin{proof}
Let $P = (t)$. The localized ring $A_P$ is a local domain with maximal ideal $tA_P$. We show $\bigcap_{n=1}^{\infty} (t^n) = (0)$ in $A$. Suppose for contradiction that a non-zero $x \in A$ satisfies $x = a_n t^n$ for all $n \geq 1$. By the definition of a Krull domain, $A$ is associated with a family of discrete valuations $\{v_i\}_{i \in I}$ on its fraction field such that $A = \{a \in K \mid v_i(a) \geq 0 \text{ for all } i\}$ and for any non-zero element, only finitely many $v_i$ take positive values. Choose a valuation $v$ with $v(t) \geq 1$. Then $v(x) = v(a_n) + n \cdot v(t) \geq n$ for every $n \geq 1$, which forces $v(x) = \infty$, contradicting $x \neq 0$. Since $A_P$ is a local domain with principal maximal ideal and $\bigcap_{n=1}^{\infty} (tA_P)^n = (0)$, it is a DVR.
\end{proof}

\section{Detailed Mathematical Examples from the Formalization}\label{appendix:detailed-examples}

This appendix provides detailed mathematical descriptions of the notable behaviors discussed in \Cref{subsec:archon-capabilities}. Each subsection corresponds to a specific observation and includes the mathematical context necessary to appreciate it.

\subsection{Autonomous gap-filling in proof details}\label{subsubsec:gap-filling}

Archon demonstrates the ability to autonomously supply non-trivial mathematical arguments that the informal proof omits. We illustrate with two examples.

The informal proof asserts that
\[
T = \mathbb{C}[[x,y,z]]/(x^2 - yz)
\]
is isomorphic to the subring $\mathbb{C}[[u^2, uv, v^2]] \subseteq \mathbb{C}[[u,v]]$ via $x \mapsto uv$, $y \mapsto u^2$, $z \mapsto v^2$. While these assignments readily define a surjection onto the subring, establishing injectivity requires a separate argument that the informal proof does not provide. Archon resolved this by explicitly constructing a quotient function at the level of power-series coefficients and verifying term-by-term that every element of the kernel is divisible by the generator $x^2 - yz$, thereby completing the injectivity proof without consulting any external reference.

A second example concerns the cardinality identity $|T| = |T/\mathfrak{m}| = |\mathbb{C}|$, which the informal proof states without elaboration. This identity is not trivial: for a power-series ring $k[[x_1, \ldots, x_n]]$, one has $|k[[x_1, \ldots, x_n]]| = |k|^{\aleph_0}$, and $\kappa^{\aleph_0} = \kappa$ does \emph{not} hold for a general infinite cardinal~$\kappa$. Archon correctly identified that the identity relies on the special property of the continuum $\mathfrak{c} = 2^{\aleph_0}$, namely $\mathfrak{c}^{\aleph_0} = (2^{\aleph_0})^{\aleph_0} = 2^{\aleph_0 \cdot \aleph_0} = 2^{\aleph_0} = \mathfrak{c}$, and applied the corresponding Mathlib lemma to close the gap.

\subsection{Handling underspecified transfinite constructions}\label{subsubsec:transfinite}

The construction of the local UFD~$A$ with prescribed completion relies on a transfinite recursive procedure due to Jensen~\cite{jensen2006completions}, building on earlier work of Loepp~\cite{loepp1997constructing} and Heitmann~\cite{heitmann1993characterization}. The source papers describe this construction at a high level---stating that one should ``recursively define $\{R_\alpha\}_{\alpha \in V}$'' and ``take unions at limit ordinals''---but leave the bookkeeping details to the reader. Formalizing such a construction in Lean~4 requires making every aspect of the recursion explicit: the data carried at each step, the well-foundedness argument, the propagation of algebraic invariants through successor and limit stages, and the precise cardinality bounds at each level.

Archon addressed this challenge by designing a bundled record type that packages each stage of the recursion together with its predecessor, the extension relation, cardinality bounds, and closure properties. It implemented the recursion via well-founded recursion on ordinals, with explicit tracking of the bound $\max(\aleph_0, |\{\gamma \le \alpha\}|)$ at each level. The limit-ordinal case was handled through a dedicated module verifying that unions of chains of N-subrings preserve the UFD property, prime extension conditions, and cardinality constraints. In effect, Archon successfully decomposed a monolithic transfinite argument into modular, independently verifiable algebraic components---isolating all technical commutative-algebra steps from the inductive scaffolding---and then assembled them into a complete formal proof. Arriving at this final architecture was itself a non-trivial process: Archon's initial attempt relied on an incorrect proof strategy, which it diagnosed and corrected through a large-scale refactoring spanning multiple sessions (Appendices~\ref{subsubsec:self-diagnosis} and~\ref{subsubsec:refactoring}).

\subsection{Self-diagnosis of mathematical errors}\label{subsubsec:self-diagnosis}

In its initial attempt at the transfinite construction, Archon employed Zorn's lemma to produce a maximal element rather than carrying out the explicit transfinite recursion required by the proof. Additionally, it conflated ``cardinality strictly less than the continuum'' with ``countable''---an identification that is valid under the Continuum Hypothesis but not provable in ZFC alone. These errors caused downstream algebraic steps to fail: the necessary cardinality bounds could not be established, and the closure properties required at limit stages did not hold.

Notably, the source papers do not explicitly warn against these pitfalls. Archon identified the root cause of the failures through its own diagnostic process: it recognized that Zorn's lemma yields only existence without the fine-grained control over intermediate stages that the construction demands, and that the countability assumption was unjustified in the absence of the Continuum Hypothesis. This self-diagnosis was carried out without human guidance and led directly to the correct reformulation described below.

\subsection{Cross-session architectural refactoring}\label{subsubsec:refactoring}

Upon recognizing the errors described in Appendix~\ref{subsubsec:self-diagnosis}, Archon undertook a large-scale refactoring of the transfinite construction, replacing the Zorn-based approach with an explicit well-founded recursion and revising all cardinality arguments to use general cardinal arithmetic in place of countability assumptions. This refactoring spanned multiple agent sessions---requiring the agent to maintain a coherent understanding of the overall proof architecture across context boundaries---and involved restructuring interdependent files throughout the Jensen construction module.

\subsection{Discovery of alternative proof strategies}\label{subsubsec:alternative-strategy}

A key technical lemma in the construction---corresponding to Lemma~4 of Heitmann~\cite{heitmann1993characterization}---establishes that a certain intersection of subrings is a UFD. The original proof proceeds by showing that the intersection is a Krull domain and then appealing to the classification of height-one primes. However, Mathlib does not currently provide a formalized definition of Krull domains, and faithfully following the original argument would have required Archon to develop substantial additional infrastructure, including the characterization of Krull domains and the theory of divisorial ideals.

Instead, Archon independently identified and applied Kaplansky's criterion---which characterizes UFDs as integral domains in which every nonzero prime ideal contains a prime element---to establish the result directly. This alternative route, which does not appear in any of the reference papers provided to the agent, bypasses the need for Krull domain theory entirely and yields a more concise formalization. This episode illustrates Archon's capacity to adapt its proof strategy to the available library support, rather than rigidly following the structure of the informal argument.

\subsection{A single hint redirecting a stalled obligation}\label{subsubsec:firstproof4-vieta}

The formalization of FirstProof Problem~4 contains the only obligation in our projects where human mathematical input redirected Archon's strategy. The original proof, at a key step, required an argument about the location of roots of a polynomial. Archon's first instinct was to discharge it through Rouch\'e's theorem (or a closely related complex-analytic root-counting argument), which is not currently developed in Mathlib. After a sustained period without progress along this analytic route, a human supervisor provided a single sentence of natural-language guidance: \emph{``Use Vieta's formulas.''} Archon adopted the suggestion, recast the step as a purely algebraic computation relating the coefficients of the polynomial to symmetric functions of its roots, and completed the obligation within an hour. In all other obligations of Problems~4 and~6, the agent either followed the route indicated by the informal proof or identified its own alternative without human intervention.

\subsection{Bypassing missing infrastructure}\label{subsubsec:firstproof4-residue}

Elsewhere in FirstProof Problem~4, Archon was required to evaluate an application of the residue theorem to a rational function with polynomial numerator and denominator. A faithful formal proof of the residue theorem in Mathlib would require building up toy contour theory, contour integration, and a formal notion of the interior of a closed curve---infrastructure that is not yet present in Mathlib; the informal source merely says ``apply the residue theorem'' without further detail. Unlike the Vieta episode (Appendix~\ref{subsubsec:firstproof4-vieta}), however, this obligation was discharged without human input. Archon observed that in the specific configuration arising in the problem---a rational function of polynomials whose numerator has degree exceeding that of the denominator---the value of interest is given directly by the classical Euler--Jacobi formula, which expresses a sum of residues as a symmetric expression in the roots of the denominator. The agent located the appropriate identity, supplied the relevant arithmetic argument, and completed the obligation entirely autonomously, without invoking any analytic machinery.\footnote{The corresponding Lean source is at \url{https://github.com/frenzymath/Archon-FirstProof-Results/blob/main/FirstProof/FirstProof4/Auxiliary/Residue.lean}.} As with the Kaplansky-criterion route in the Anderson formalization (Appendix~\ref{subsubsec:alternative-strategy}), this alternative path does not appear in any reference material consulted by the agent and reflects Archon's adaptive response to library gaps.

\section{Statement Verification via Comparator}\label{appendix:szm_comparator}

As described in \Cref{subsec:overview-result}, we verify the correctness of the formalization at two levels. First, the entire project passes \texttt{lake build}, confirming that all proof obligations are discharged. Second, we use Comparator~\cite{comparator} to confirm that the theorem proved in the full project matches a short, human-readable specification and relies only on the standard Lean axioms.

The specification file, \texttt{Challenge.lean}, is reproduced below. It contains only the definitions of quasi-completeness and weak quasi-completeness, together with the statement of the main theorem. A human reviewer can verify the semantic correctness of these 30 lines without any knowledge of the proof or the rest of the codebase. Comparator then automatically checks that the \texttt{main\_theorem} declaration in the full project proves exactly this statement, with no hidden axioms.

\begin{Verbatim}[fontsize=\small,commandchars=!?\|]
import Mathlib.Algebra.Algebra.Operations
import Mathlib.RingTheory.LocalRing.MaximalIdeal.Defs
import Mathlib.RingTheory.Noetherian.Defs

variable (R : Type*) [CommRing R] [IsLocalRing R]

/-- A local ring `R` is quasi-complete if for any antitone sequence
of ideals `A : !(!mathbb?N|!) → Ideal R` and each `k : !(!mathbb?N|!)`, there exists `s`
such that `A s !(!le!) (!(!sqcap!) n, A n) !(!sqcup!) (IsLocalRing.maximalIdeal R) ^ k`.

This is Definition 1.1 of Anderson (2014). -/
def IsQuasiComplete : Prop :=
  !(!forall!) (A : !(!mathbb?N|!) → Ideal R), Antitone A →
    !(!forall!) (k : !(!mathbb?N|!)), !(!exists!) s,
      A s !(!le!) (!(!sqcap!) n, A n) !(!sqcup!) (IsLocalRing.maximalIdeal R) ^ k

/-- A local ring `R` is weakly quasi-complete if for any antitone
sequence of ideals `A : !(!mathbb?N|!) → Ideal R` with `!(!sqcap!) n, A n = !(!bot!)` and
each `k : !(!mathbb?N|!)`, there exists `s` such that
`A s !(!le!) (IsLocalRing.maximalIdeal R) ^ k`. -/
def IsWeaklyQuasiComplete : Prop :=
  !(!forall!) (A : !(!mathbb?N|!) → Ideal R), Antitone A → (!(!sqcap!) n, A n) = !(!bot!) →
    !(!forall!) (k : !(!mathbb?N|!)), !(!exists!) s,
      A s !(!le!) (IsLocalRing.maximalIdeal R) ^ k

/-- **Main Theorem**: There exists a weakly quasi-complete
Noetherian local ring that is not quasi-complete. -/
theorem main_theorem :
    !(!exists!) (R : Type) (_ : CommRing R) (_ : IsLocalRing R)
      (_ : IsNoetherianRing R),
      IsWeaklyQuasiComplete R !(!wedge!) ¬ IsQuasiComplete R := by
  sorry
\end{Verbatim}
\noindent The \texttt{sorry} in this file is intentional: it marks an obligation the full project must discharge. Comparator verifies that the project's \texttt{main\_theorem} fills this \texttt{sorry} with a complete proof, matching the statement exactly.

\end{document}